%% file: arXiv2.tex
\documentclass[12pt]{article}

\usepackage[margin=1in]{geometry}

\usepackage{setspace}

\usepackage{graphicx}

\usepackage{enumitem}

\usepackage{microtype}
\usepackage{subfigure}
\usepackage{booktabs}
\usepackage{comment}
\usepackage{wrapfig}
\usepackage{arydshln}
\usepackage{enumitem}
\usepackage{multirow}
\usepackage{url}
\usepackage{natbib}
\usepackage{authblk}

\RequirePackage[colorlinks,citecolor=blue,linkcolor=blue,urlcolor=blue,pagebackref]{hyperref}

\usepackage{amsmath}
\usepackage{amssymb}
\usepackage{mathtools}
\usepackage{amsfonts}
\usepackage{bm}
\usepackage{bbm}

\input{Macros/algorithm.tex}
\input{Macros/math.tex}
\input{Macros/theorems.tex}

\input{Macros/misc.tex}

\usepackage[capitalize,noabbrev]{cleveref}

\allowdisplaybreaks

\input{Head/Title}
\input{Head/AuthorInf}

\begin{document}

\maketitle

\input{Main.tex}

\end{document}

%% file: Macros/algorithm.tex
\usepackage{algorithm}
\usepackage{algorithmic}

%% file: Macros/math.tex
\def\1{\bm{1}}

\DeclareMathAlphabet{\mathsfit}{\encodingdefault}{\sfdefault}{m}{sl}
\SetMathAlphabet{\mathsfit}{bold}{\encodingdefault}{\sfdefault}{bx}{n}

\def\calD{{\mathcal{D}}}

\def\calT{{\mathcal{T}}}

\newcommand{\E}{\mathbb{E}}

\newcommand{\R}{\mathbb{R}}

\newcommand{\Var}{\mathrm{Var}}

\newcommand{\Cov}{\mathrm{Cov}}

\newcommand{\Prb}[1]{\mathbb{P}\left[#1\right]}


%% file: Macros/theorems.tex
\usepackage{amsthm}

\theoremstyle{plain}

\newtheorem{theorem}{Theorem}[section]

\newtheorem{corollary}[theorem]{Corollary}
\newtheorem{proposition}[theorem]{Proposition}

\newtheorem{assumption}{Assumption}[section]

%% file: Macros/misc.tex
\usepackage[textsize=tiny]{todonotes}
\usepackage{multirow}
\usepackage{wrapfig}
\usepackage{subfigure}
\usepackage{tabularx}

\usepackage{xcolor}
\newcount\Comments  
\newcommand{\kibitz}[2]{\ifnum\Comments=1\textcolor{#1}{#2}\fi}

\usepackage{listings}

\usepackage{tabularx}

%% file: Head/Title.tex
\title{Generative Marketing Mix Modeling: A Causal Inference Framework Linking GEO and GEM to Business Impact}

%% file: Head/AuthorInf.tex
\author[1]{Masahiro Kato\thanks{Email: \texttt{mkato-csecon@g.ecc.u-tokyo.ac.jp}}$\,$}
\author[2]{Daiki Honma}
\author[2]{Taka Kato}

\affil[1]{The University of Tokyo and Mizuho-DL Financial Technology Co., Ltd.}
\affil[2]{NP-hard}

\date{\today}

%% file: Main.tex
\begin{abstract}
Generative artificial intelligence changes how firms reach customers, but standard marketing data do not record how often users see and notice a firm's name in generated answers. We develop Generative Marketing Mix Modeling (GMMM) to estimate the causal effects of Generative Engine Optimization (GEO) and Generative Engine Marketing (GEM). For GEO, GMMM combines repeated generated answers with question counts, shares of use across generative systems, and notice probabilities. For GEM, it combines records of sponsored placements with notice probabilities. GMMM compares expected business responses under alternative treatment sequences and establishes sufficient conditions for identifying the resulting effects. We investigate the empirical performance of the proposed method using simulated answers to product recommendation in English and Japanese. 
\end{abstract}


\section{Introduction}
Marketing mix modeling (MMM) relates an aggregate response, such as sales or conversions, to media inputs observed across markets and periods. Two features of advertising motivate the transformations used in MMM. An advertisement can affect the response after the period in which it is shown, so a carryover function combines current and past inputs. The marginal change in response can also become smaller when the accumulated input is already high, so a saturation function maps the carried-over input into a bounded or slowly increasing regressor. MMM applies these transformations before estimating the coefficients of the media channels \citep{Nerlove1962optialadvertizing,Clarke1976econometricmeasurement,Hanssens2001marketresponse}.

Generative artificial intelligence (AI) changes what must be measured before this model can be used. Through Generative Engine Optimization (GEO), a firm modifies source material that may alter answers about the firm, for example by adding a frequently asked questions section or revising product documentation. Through Generative Engine Marketing (GEM), a firm pays for sponsored placements in generated answers \citep{Aggarwal2024geogenerative,Feizi2026onlineadvertisements,Hu2026gembencha}. Conventional impression logs do not record nonsponsored occurrences of a firm's name in generated answers, and a platform record of sponsored placements does not show whether users noticed them.

The frequency with which collected answers contain a firm's name is not yet a media input for a market and period. To obtain an expected count for a market and period, that frequency must be combined with the number of relevant questions, the share handled by each generative system, and the probability that a user notices the name. Referral sessions measure a different event because users may read an answer without following a link. For sponsored placements, spending must likewise be related to the number of placements shown and the probability of notice. GEO and GEM require this additional measurement step before they can be analyzed alongside media channels whose impressions are recorded directly.

Estimating a treatment effect adds a second requirement. Carryover makes the response in period $t$ depend on inputs from earlier periods, so disabling GEO or GEM changes the entire sequence of media inputs over the affected periods. The relevant comparison must replace the full sequence under one treatment with the full sequence under another treatment before the response is evaluated. A regression coefficient by itself does not perform this comparison. Throughout this study, GEO denotes the specified source modification, and its treatment effect is the total effect of applying that modification. The response may change through the generated answers represented by the GEO input and through other consequences of the modification, such as conventional search traffic or conversion on a landing page.

\subsection{Contributions}
We develop GMMM to connect measurements of generated answers and sponsored placements to MMM. For GEO, GMMM models the probability that an answer contains a specified feature and combines that probability with market query counts and notice probabilities to obtain the expected number of noticed occurrences under each source state. This count can be positive without the source modification because GEO changes the occurrence probability rather than creating the entire input from zero. For GEM, GMMM relates spending to the number of sponsored placements and then accounts for notice. The resulting inputs receive the same carryover and Hill transformations as established media before they enter the response model.

Estimation need not be Bayesian. A likelihood or a penalized criterion can support plug-in, joint, and other regularized procedures. We use a Bayesian formulation to average over uncertainty in the occurrence and notice probabilities. The cut posterior leaves the distribution of the measurement parameters determined by the measurement data, whereas the joint posterior also allows the response data to update it. Comparing the two procedures shows how estimation of the media transformations and uncertainty about the constructed inputs affect the treatment effect of GEO.

The identification analysis gives conditions under which the treatment effect remains identifiable even when the coefficients on the GEO source state and the GEO input cannot be determined separately after the media transformations are fixed. It also characterizes when replacing the expected market count by a sequence of occurrence probabilities changes only the scale and when changes in market composition prevent that reduction.

The empirical collection contains 2,240 complete answers from GPT-5.6 Luna and GPT-4o. Each model answers the same set of 56 questions, comprising 28 English questions and 28 Japanese questions. Every call uses the same system instruction and requires web search. The target name is not included in the system instruction or in the questions. Glasp occurs in $33.8\%$ of the GPT-5.6 Luna answers and $27.8\%$ of the GPT-4o answers, although GPT-4o has the higher rate for the English questions. These observations estimate occurrence probabilities under the source state present during collection. One simulation design uses those estimates for the baseline occurrence probabilities, while the simulation itself generates the treatment effect. Across the controlled comparisons, estimating carryover and saturation explains most of the improvement in coefficient recovery over a plug-in method that fixes them, while updating the measurement parameters with response data has no uniform advantage for estimating the treatment effect.

\subsection{Related Work}
Classical MMM combines distributed lags with nonlinear response functions, and recent methods use regularization or Bayesian pooling across markets \citep{Jin2017bayesianmethods,Ng2024bayesiantime,Runge2025packagingup,Gong2024causalmmmlearning,Sun2017geolevelbayesian}. These methods estimate a response model once its media inputs have been defined. Causal interpretation requires further evidence because predictive fit alone does not identify an advertising effect. Work on incrementality and geographic experiments studies the assignment mechanisms and randomized comparisons that can support such an interpretation \citep{Chan2017challengesand,Lewis2015theunfavorable,Gordon2019acomparison,Vaver2011measuringad,Chen2022robustcausal}.

Research on GEO studies how changes in source material affect generated answers and their ranking, while work on GEM considers sponsored retrieval and placement \citep{Bagga2026egeoa,Kim2026sageoarena,Martinez2026optimizingvisibility,Hajiaghayi2024adauctions,Dubey2024auctionswith,Dutting2024mechanismdesign,Xu2026adinsertion}. Studies of GEO treat citation of a source and use of its content in an answer as different outcomes \citep{Zhang2026fromcitation}. Neither quantity alone gives the number of users who notice the target feature. GMMM addresses the next step by placing the occurrence probability on a market and time scale suitable for MMM.

The use of records on reach and frequency in MMM provides a precedent for replacing spending with quantities closer to exposure \citep{Zhang2023bayesianhierarchical}. Randomized estimates have likewise been incorporated through informative priors or likelihood terms \citep{Zhang2024mediamix}. Meridian and PyMC-Marketing are Bayesian implementations of MMM, and Meridian supports inputs based on reach and frequency.\footnote{See \url{https://developers.google.com/meridian} and \url{https://www.pymc-marketing.io/}.} Our contribution concerns the measurements required when the media input itself must be inferred from generated answers or sponsored placements.

Cut and joint posteriors arise in modular Bayesian inference \citep{Plummer2015cutsin,Jacob2017bettertogether,Carmona2020semimodularinference}. A cut posterior passes uncertainty from one module to another without allowing the later module to revise the first distribution. Its interpretation in GMMM depends on the assignment of parameters to modules: fixing the prior for the media transformations also prevents their estimation from response data, whereas cutting the update of the measurement parameters still permits those transformations to be estimated.

\section{Setup}
\label{sec:setup}
We use \emph{response} for the business variable analyzed by MMM and \emph{answer} for text returned by a generative system. A \emph{treatment} is a specified setting of GEO or GEM. When several periods are involved, a treatment sequence lists those settings from period $1$ through period $T$ and includes any earlier values needed to initialize carryover.

\subsection{Response and Media Inputs in MMM}
For every observational unit $i\in\{1,\ldots,n\}$ and period $t\in\{1,\ldots,T\}$, let $Y_{it}$ denote the response and let $W_{it}$ contain baseline variables and controls observed before the period-$t$ treatment. The unit may be a geographic market or another aggregation used consistently in the measurement and response models. Let $\mathcal M_0$ be the set of established media channels. For every $m\in\mathcal M_0$, let $E_{mit}\geq0$ denote the media input in period $t$ and define its sequence through period $t$ by $\boldsymbol E_{mi,1:t}=(E_{mi1},\ldots,E_{mit})$.

A standard MMM transforms each sequence of media inputs before including it in the response model. We write
\begin{align}
H_{mit}
&=
h_m\left(
A_m(\boldsymbol E_{mi,1:t};\alpha_m);\theta_m
\right),
\label{eq:general-media-transform}
\end{align}
where $A_m$ combines current and earlier inputs according to carryover parameters $\alpha_m$, and $h_m$ represents saturation with parameters $\theta_m$. Let $b(W;\gamma)$ be the baseline response with coefficient vector $\gamma$. If $\mathcal I_0$ is a specified set of distinct pairs of media channels, the conditional mean satisfies
\begin{align}
g_Y\left(\mu^{\mathrm{MMM}}_{it}\right)
&=
b(W_{it};\gamma)
+
\sum_{m\in\mathcal M_0}\beta_mH_{mit}
+
\sum_{(m,m')\in\mathcal I_0}
\beta_{mm'}H_{mit}H_{m'it},
\label{eq:general-mmm}
\end{align}
where $\mu^{\mathrm{MMM}}_{it}=\E[Y_{it}\mid\mathcal H_{it}]$. The information set $\mathcal H_{it}$ contains the controls and the sequences of media inputs used to model $Y_{it}$. We specify $Y_{it}\mid\mathcal H_{it}\sim\mathcal F(\mu^{\mathrm{MMM}}_{it},\psi)$, where $\psi$ contains the remaining parameters of the response distribution. The identity link gives the usual model for a continuous response with additive noise of mean zero, while other links accommodate counts or rates.

\subsection{Channels Through Generative AI}
GMMM enlarges the media set to $\mathcal M=\mathcal M_0\cup\{G,P\}$, where $G$ indexes the input constructed from generated answers and $P$ indexes the input constructed from sponsored placements. For each period, the first input is the expected number of generated answers in which the specified property occurs and is noticed under the current source state. The second is the expected number of sponsored placements that users notice. The first count need not be zero when the source modification is absent; GEO changes its value by changing generated answers. Inputs for established media are observed directly, while the two additional inputs are constructed from generated answers, platform records, market counts, and observations of user attention.

Let $Z_{it}\in\{0,1\}$ denote the GEO source state, where $Z_{it}=1$ means that the specified source modification is present and $Z_{it}=0$ means that it is absent. The response specification studied here is
\begin{align}
g_Y\left(\mu_{it}\right)
={}&
b(W_{it};\gamma)
+
\sum_{m\in\mathcal M_0}\beta_mH_{mit}
+
\beta_D Z_{it}
+
\beta_GH^G_{it}
+
\beta_PH^P_{it}
+
\beta_{GP}H^G_{it}H^P_{it}.
\label{eq:gmmm-response}
\end{align}
The term $\beta_D Z_{it}$ permits the source modification to affect the response through variables that are not represented by the constructed GEO input, such as conventional search traffic or conversion on a landing page. The coefficient $\beta_{GP}$ permits the effect associated with one generative channel to depend on the other. Either term may be omitted when the corresponding mechanism is excluded from the model.

\subsection{Data Sources}
Let $\calD_M$ contain the observations used to construct the GEO and GEM inputs. For GEO, these observations include repeated generated answers, counts of relevant questions, shares of use across generative systems, and observed notice indicators. For GEM, they include spending, the number of sponsored placements shown, predictors of that number, and observed notice indicators. Let $\calD_Y$ contain $Y_{it}$ together with the controls and media inputs in the response model. When a randomized experiment estimates the same treatment effect for a comparable population and evaluation period, its estimate and standard error form $\calD_E$.

The collection described in Section~\ref{sec:answer-collection} contains 20 answers from each of two GPT models for each of 56 questions, comprising 28 English questions and 28 Japanese questions. The public referral series analyzed in Appendix~\ref{sec:empirical} comes from an earlier period and lacks the corresponding question counts, generated answers, and notice observations. These missing observations prevent the two sources from being combined in one GMMM analysis.

The simulations use units of the form $i=(g,j)$, where $g\in\{1,\ldots,5\}$ indexes geographic markets and $j\in\{1,\ldots,6\}$ indexes product clusters. Each cluster contains pages and questions that respond to the same GEO treatment. A substantive application must use an observational unit that is common to the construction of the media inputs, the response model, and the treatment effect.

\subsection{Treatment Effect}
Because carryover links the period-$t$ response to media inputs observed before $t$, we define the treatment effect from complete treatment sequences over a specified evaluation window.

Let $a_G\in\{0,1\}$ select one of two GEO sequences: $a_G=1$ selects the specified sequence with GEO, and $a_G=0$ selects the sequence with the source modification removed. Let $a_P\in\{0,1\}$ similarly select the specified GEM spending sequence or the sequence with GEM removed. The pair $(a_G,a_P)$ selects one complete GEO sequence and one complete GEM sequence across all periods. The comparison uses the same history before period $1$ unless an intervention begins earlier; in that case, the treatment and input values before period $1$ follow the selected sequence. The simulations set all media inputs before period $1$ to zero. For an evaluation set $\calT_E\subseteq\{1,\ldots,T\}$, define
\begin{align}
V(a_G,a_P)
&=
\sum_{i=1}^{n}
\sum_{t\in\calT_E}
\E\left[Y_{it}(a_G,a_P)\right],
\label{eq:business-value}
\end{align}
where $Y_{it}(a_G,a_P)$ is the potential response under the selected treatment sequences and the specified evolution of all other variables. The primary estimand is
\begin{align}
\Delta_G
&=
V(1,1)-V(0,1).
\label{eq:geo-treatment-effect}
\end{align}
We refer to $\Delta_G$ as the treatment effect of GEO. It is the total effect of the specified source modification, including the change transmitted through $H^G$ and any additional change represented by $\beta_DZ_{it}$. The contrast retains the specified GEM sequence. The evaluation window may extend beyond the periods in which GEO is active so that the effect includes responses associated with carried-over inputs. The analogous treatment effect of GEM is $\Delta_P=V(1,1)-V(1,0)$, which compares the specified GEM sequence with the sequence in which GEM is disabled while retaining the GEO sequence.

\section{Generative Marketing Mix Modeling}
\label{sec:method}
The data record the source state, spending, and established media inputs, but they omit the expected counts associated with GEO and GEM. GMMM constructs these two counts before applying the standard media transformations and estimating the response model.

\subsection{The GEO Input}
Let $q\in\{1,\ldots,Q\}$ index clusters of questions and let $p\in\mathcal P$ index generative systems. A system is defined by the model version and the settings used to obtain an answer. The set $\mathcal Q(i)$ contains the question clusters associated with unit $i$. For every $i$, $q$, and $t$, let $N_{iqt}\geq0$ be the number of times that a question in cluster $q$ is submitted in unit $i$ during period $t$. For every $i$, $p$, and $t$, let $\omega_{ipt}\in[0,1]$ be the share of those questions handled by system $p$, with $\sum_{p\in\mathcal P}\omega_{ipt}=1$.

Before collecting answers, the analyst specifies the property to be recorded. For repetition $r\in\{1,\ldots,R_{qpt}\}$ and source state $z\in\{0,1\}$, let $X_{qptr}(z)=1$ when the stored answer has that property and let it equal zero otherwise. In the empirical collection, the property is the occurrence of a spelling of the target name in a fixed dictionary. Let
$\pi_{qpt}(z)=\Prb(X_{qptr}(z)=1)$ be its occurrence probability. Conditional on $X_{qptr}(z)=1$, let $\lambda_{qpt}\in[0,1]$ be the probability that the user notices the recorded property. The input associated with GEO is
\begin{align}
E^G_{it}(z)
&=
\sum_{q\in\mathcal Q(i)}
N_{iqt}
\sum_{p\in\mathcal P}
\omega_{ipt}\pi_{qpt}(z)\lambda_{qpt}.
\label{eq:geo-exposure}
\end{align}
The quantity $E^G_{it}(z)$ is the expected number of generated answers in unit $i$ and period $t$ for which the specified property occurs and is noticed under source state $z$. It may be positive when $z=0$; the change in this input caused by the source modification is $E^G_{it}(1)-E^G_{it}(0)$. A user who encounters the property more than once contributes more than once to this count. The comparison between $z=1$ and $z=0$ holds $N_{iqt}$ and $\omega_{ipt}$ fixed. We assume that the source modification does not change notice conditional on occurrence, which is why $\lambda_{qpt}$ has no argument $z$.

Equation~\eqref{eq:geo-exposure} pools occurrence and notice probabilities across markets after conditioning on the question cluster, generative system, period, and source state. It also uses a system share that is common across question clusters within a market and period. When the data distinguish these cells, the same construction can use $\pi_{iqpt}(z)$, $\lambda_{iqpt}$, and $\omega_{iqpt}$ without changing the media transformations or the treatment contrasts.

We estimate the occurrence probabilities with the hierarchical model
\begin{align}
X_{qptr}\mid\pi_{qpt}
&\sim\operatorname{Bernoulli}(\pi_{qpt}),
\label{eq:occurrence-binomial}\\
\operatorname{logit}(\pi_{qpt})
&=
A_{qpt}^{\mathsf T}\alpha+u_q,
\qquad
\boldsymbol u=B_Q\sigma_u\widetilde{\boldsymbol u},
\qquad
\widetilde{\boldsymbol u}\sim\mathcal N(0,I_{Q-1}).
\label{eq:occurrence-logit}
\end{align}
The columns of $B_Q\in\R^{Q\times(Q-1)}$ form an orthonormal basis for vectors whose coordinates sum to zero, so $\sum_{q=1}^{Q}u_q=0$. The predictor vector $A_{qpt}$ may contain indicators for the generative system, the source state, and other observed determinants of occurrence. Its coefficient vector is $\alpha$. The effect for question cluster $q$ is represented in noncentered form by standard normal coordinates $\widetilde{\boldsymbol u}$ and scale $\sigma_u>0$. Replacing the observed sequence of source states by another treatment sequence yields probabilities under that sequence only when the data contain the required variation in source state. Answers collected under one source state identify only the occurrence probabilities for that state; the change caused by the source modification requires observations under both states.

Suppose that a user study presents generated answers and records $C_{qpt}$ cases in which participants notice the property among $n_{qpt}$ answers in which it occurs. We use
\begin{align}
C_{qpt}\mid\lambda_{qpt}
&\sim\operatorname{Binomial}(n_{qpt},\lambda_{qpt}),
\label{eq:notice-binomial}\\
\operatorname{logit}(\lambda_{qpt})
&=
(D^\lambda_{qpt})^{\mathsf T}\xi+v_q,
\qquad
\boldsymbol v=B_Q\sigma_v\widetilde{\boldsymbol v},
\qquad
\widetilde{\boldsymbol v}\sim\mathcal N(0,I_{Q-1}).
\label{eq:notice-logit}
\end{align}
Here, $D^\lambda_{qpt}$ is a vector of observed predictors with coefficient vector $\xi$. The effect for question cluster $q$, denoted by $v_q$, uses the same zero-sum basis and its own scale $\sigma_v>0$. Together, the two models estimate the probabilities needed in \eqref{eq:geo-exposure}.

\subsection{The GEM Input}
A sponsored placement contributes to the GEM input only if the platform shows it and the user notices it. Let $S^P_{it}\geq0$ denote the spending level selected by the firm for the GEM intervention in unit $i$ and period $t$, before the platform determines delivery. Let $L^P_{it}$ denote the number of sponsored placements shown. Conditional on a placement being shown, let $\rho^P_{it}\in[0,1]$ be the probability that the user notices it. For positive spending, we model the number shown by
\begin{align}
L^P_{it}\mid S^P_{it}>0
&\sim\operatorname{Poisson}(\mu^P_{it}),
\label{eq:paid-placement-distribution}\\
\log\mu^P_{it}
&=
\phi_0
+
\phi_S\log S^P_{it}
+
\phi_DD_{it}
+
\phi_RR_{it},
\qquad
\phi_S>0,
\label{eq:paid-placement-mean}
\end{align}
where $D_{it}$ is a demand predictor observed before the period-$t$ treatment and $R_{it}$ is a promotion indicator. We set $\mu^P_{it}=0$ when $S^P_{it}=0$. The input associated with GEM is
\begin{align}
E^P_{it}
&=
\mu^P_{it}\rho^P_{it}.
\label{eq:paid-exposure}
\end{align}
The restriction $\phi_S>0$ makes the expected number of sponsored placements increase with positive spending. A user study in which participants report whether they noticed each sponsored placement can estimate $\rho^P_{it}$ with a binomial model. The analysis below uses one notice probability for all sponsored placements, while \eqref{eq:paid-exposure} permits variation across units and periods when the data support it. When $L^P_{it}$ is observed and the analysis conditions on realized delivery, $L^P_{it}\rho^P_{it}$ can be used directly. The Poisson model is needed for missing delivery counts and for spending sequences that were not observed.

\subsection{Carryover and Saturation}
The two expected counts are transformed in the same way as inputs for established media. We use normalized geometric carryover,
\begin{align}
A_{mit}
&=
\frac{
\sum_{\ell=0}^{L_m}
\alpha_m^{\ell}E_{mi,t-\ell}/s_m
}{
\sum_{\ell=0}^{L_m}\alpha_m^{\ell}
},
\qquad
0\leq\alpha_m<1,
\label{eq:adstock}
\end{align}
where $L_m\in\{0,1,\ldots\}$ is the maximum lag and $s_m>0$ is a fixed scale used in the analysis. Values before period $1$ come from the earlier history assigned to the treatment sequence; the simulations set these values to zero. The transformed input is
\begin{align}
H_{mit}
&=
\frac{A_{mit}^{\kappa_m}}
{A_{mit}^{\kappa_m}+\theta_m^{\kappa_m}},
\qquad
\theta_m>0,
\qquad
\kappa_m>0.
\label{eq:hill}
\end{align}
The Hill function equals $1/2$ when $A_{mit}=\theta_m$, and $\kappa_m$ determines how sharply it changes near that point. Carryover is applied before the Hill transformation \citep{Jin2017bayesianmethods}. Substituting $H^G_{it}$ and $H^P_{it}$ into \eqref{eq:gmmm-response} completes the response specification.

\subsection{Estimation from Measurement and Response Data}
Let $\eta_M$ collect the parameters used in \eqref{eq:geo-exposure} and \eqref{eq:paid-exposure}, let $\eta_T$ collect the carryover and Hill parameters, and let $\vartheta$ collect the response parameters. We write $\mathcal J_M(\eta_M;\calD_M)$ for a loss based on the data used to construct the two inputs and $\mathcal J_Y(\vartheta,\eta_M,\eta_T;\calD_Y)$ for a loss based on the response data. If $\calD_E$ contains a randomized estimate of the same treatment effect, its contribution is $\mathcal J_E(\vartheta,\eta_M,\eta_T;\calD_E)$; otherwise, we set $\mathcal J_E=0$. A general estimator minimizes
\begin{align}
\mathcal J(\eta_M,\eta_T,\vartheta)
&=
\mathcal J_M
+
\mathcal J_Y
+
\mathcal J_E
+
\mathcal P_M(\eta_M)
+
\mathcal P_T(\eta_T)
+
\mathcal P_Y(\vartheta),
\label{eq:general-criterion}
\end{align}
where the penalty terms may be zero. With negative log likelihoods and no penalties, this criterion gives maximum likelihood estimation. Nonzero penalties permit regularization.

The criterion covers several ways of using the data. A plug-in method estimates $\eta_M$ from $\calD_M$, substitutes that estimate into the response model, and either fixes $\eta_T$ or estimates it from $\calD_Y$. A joint likelihood estimates $\eta_M$, $\eta_T$, and $\vartheta$ together. Resampling can be used with either procedure to assess uncertainty. A Bayesian version makes both the treatment of uncertainty and the direction of updating explicit.

\subsection{Bayesian Computation}
Minimizing the criterion with negative log likelihoods gives point estimates. In the Bayesian analysis, the cut posterior leaves the distribution of the measurement parameters determined by $\calD_M$, whereas the joint posterior allows $\calD_Y$ to update it. Both procedures average over uncertainty in the constructed inputs. The hierarchical occurrence and notice models are fitted in noncentered coordinates. Their posterior modes and analytic Hessians define a Laplace approximation $q_L(\eta_M\mid\calD_M)$ to the posterior based only on $\calD_M$. Sampling the transformation parameters from their prior gives
\begin{align}
q(\eta\mid\calD_M)
&=
q_L(\eta_M\mid\calD_M)p(\eta_T),
\qquad
\eta=(\eta_M,\eta_T).
\label{eq:measurement-distribution}
\end{align}
Let $\beta=(\beta_D,\beta_G,\beta_P,\beta_{GP})^{\mathsf T}$. After replacing the exact posterior for $\eta_M$ by the Laplace approximation, the joint posterior is
\begin{align}
\widetilde p(\eta,\beta,\sigma_Y^2\mid\calD_M,\calD_Y)
&\propto
p(\calD_Y\mid\eta,\beta,\sigma_Y^2)
p(\beta,\sigma_Y^2)
q(\eta\mid\calD_M).
\label{eq:joint-posterior}
\end{align}
For each $k\in\{1,\ldots,K\}$, we sample $\eta_k\sim q(\eta\mid\calD_M)$ and construct the corresponding sequences of media inputs. Conditional on $\eta_k$, the response parameters are integrated under the normal inverse-gamma model with the sign restrictions in Appendix~\ref{sec:response-posterior}. Let $m_k$ be the numerical approximation to $p(\calD_Y\mid\eta_k)$ obtained with a Gaussian approximation to the posterior probability of the sign restrictions. The joint calculation assigns weight
\begin{align}
w_k
&=
\frac{m_k}{\sum_{h=1}^{K}m_h}.
\label{eq:importance-weights}
\end{align}
After selecting $\eta_k$ according to these weights, we sample the response parameters from their conditional posterior and evaluate $\Delta_G$.

The two-stage procedure uses equal weights for the samples from \eqref{eq:measurement-distribution}, leaving both the Laplace approximation for $\eta_M$ and the prior for $\eta_T$ unchanged by $\calD_Y$. A cut posterior keeps only the first of these distributions fixed:
\begin{align}
p_{\mathrm{cut}}(\eta_M,\eta_T,\beta,\sigma_Y^2\mid\calD_M,\calD_Y)
&=
q_L(\eta_M\mid\calD_M)
p(\eta_T,\beta,\sigma_Y^2\mid\calD_Y,\eta_M).
\label{eq:cut-posterior}
\end{align}
Because the conditional posterior on the right is normalized separately for every $\eta_M$, integrating over $(\eta_T,\beta,\sigma_Y^2)$ leaves $q_L(\eta_M\mid\calD_M)$ unchanged \citep{Plummer2015cutsin,Jacob2017bettertogether}.

For computation, we use $M$ samples of the measurement parameters and a common set of $J$ samples from the transformation prior. Let $m_{mj}$ approximate $p(\calD_Y\mid\eta_{M,m},\eta_{T,j})$. On this common set, the cut and joint weights are
\begin{align}
w^{\mathrm{cut}}_{mj}
&=
\frac{1}{M}
\frac{m_{mj}}{\sum_{h=1}^{J}m_{mh}},
&
w^{\mathrm{joint}}_{mj}
&=
\frac{m_{mj}}
{\sum_{r=1}^{M}\sum_{h=1}^{J}m_{rh}}.
\label{eq:modular-weights}
\end{align}
The plug-in method that estimates the transformations uses the same $J$ samples at the point estimate of $\eta_M$, while the two-stage procedure assigns weight $1/(MJ)$ to every pair. For the independent samples in \eqref{eq:importance-weights}, we report
\begin{align}
\operatorname{ESS}
&=
\left(\sum_{k=1}^{K}w_k^2\right)^{-1}
\label{eq:ess}
\end{align}
as a numerical diagnostic. A small effective sample size means that a few samples carry most of the weight. For the common set of samples, Appendix~\ref{sec:modular-diagnostics} reports the effective sample size within each measurement sample and the effective sample size of the joint weights across measurement samples.

Figure~\ref{fig:framework} summarizes the order of the calculations.
\begin{figure}[!htbp]
\centering
\includegraphics[width=\textwidth,draft=false]{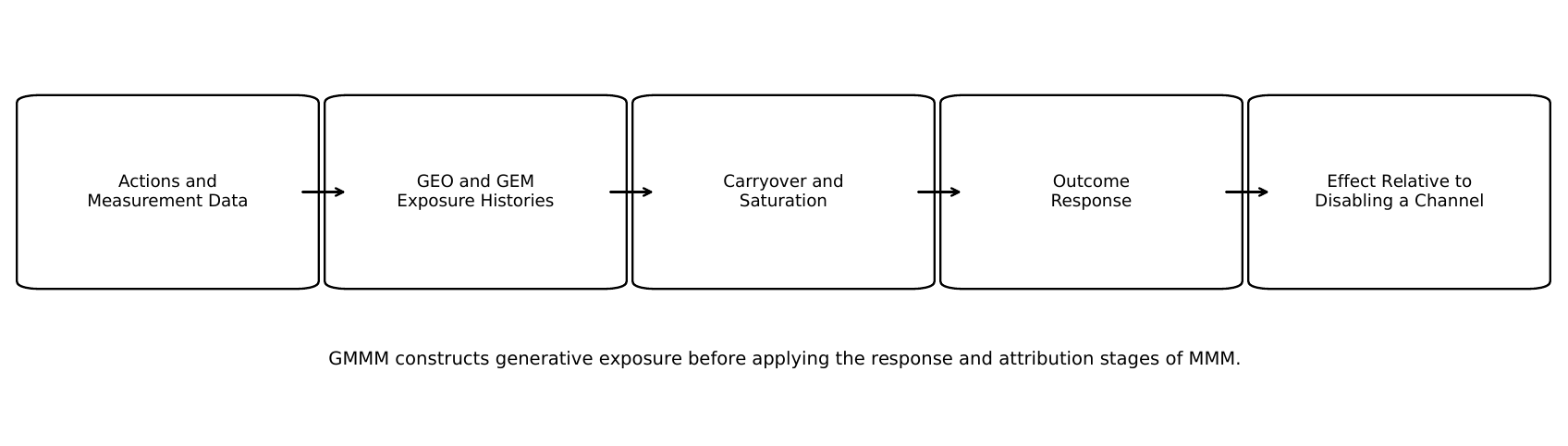}
\caption{GMMM constructs the expected counts associated with GEO and GEM before applying carryover, the Hill transformation, and the response model. The plug-in, cut, and joint procedures differ in whether they average over uncertainty in the measurement parameters and whether the response data update those parameters.}
\label{fig:framework}
\end{figure}

\subsection{Randomized Estimate of the GEO Treatment Effect}
A randomized experiment can inform GMMM when it estimates the same treatment effect for a population, response definition, and evaluation window that can be related to the GMMM target. Suppose that the experiment reports $\widehat\tau_E$ with standard error $s_E$, and let $\tau_E(\vartheta,\eta)$ be the corresponding effect implied by the GMMM parameters. We use
\begin{align}
\widehat\tau_E\mid\vartheta,\eta
&\sim
\mathcal N\left(
\tau_E(\vartheta,\eta),
s_E^2+\sigma_{\mathrm{tr}}^2
\right),
\label{eq:randomized-estimate-likelihood}
\end{align}
where $\sigma_{\mathrm{tr}}\geq0$ is fixed before estimation and represents residual differences between the experimental population and the target population after the measured design variables have been aligned. This likelihood contributes $\mathcal J_E$ to \eqref{eq:general-criterion}. In the Bayesian analysis, it changes the weights and the conditional posterior for the response coefficients. For every parameter sample, the experimental effect is computed from the complete treatment and control sequences before the media transformations are applied.

The aligned simulation uses the average treatment effect of GEO, including the term associated with the source state. A second simulation adds a nonzero difference between the experimental and target populations to the same effect and thereby examines misspecification of the transport model. If an additive attribution across interacting channels is also required, Appendix~\ref{sec:interaction-allocation} gives a Shapley allocation as an additional summary.

\section{Identification of the Treatment Effect of GEO}
\label{sec:theory}
Recovering the treatment effect of GEO requires both statistical variation in the response model and observations that support the expected response under each treatment sequence. Statistical variation determines whether the response coefficients or the linear combination that defines the effect can be identified, while the observed market inputs and treatment assignments determine whether the comparison can be evaluated.

\subsection{Identification Within the Response Model}
If the two constructed inputs were observed directly, GMMM would differ from a standard MMM only through the additional term for the GEO source state. The following reduction states this relation before the rank condition for the response coefficients.

\begin{proposition}[Reduction to standard MMM]
Suppose that $\boldsymbol E^G_{i,1:T}$ and $\boldsymbol E^P_{i,1:T}$ are observed for every $i\in\{1,\ldots,n\}$. If $\beta_D=0$, then \eqref{eq:gmmm-response} is an instance of \eqref{eq:general-mmm} with channel set $\mathcal M_0\cup\{G,P\}$ and interaction set $\{(G,P)\}$. If $\beta_{GP}=0$ also holds, the response is additive on this enlarged channel set. Setting all four coefficients associated with GEO and GEM to zero gives the additive model for established media with $\mathcal I_0=\varnothing$.
\end{proposition}
\begin{proof}
The observed sequences determine $H^G$ and $H^P$ through \eqref{eq:general-media-transform}. Substitution into \eqref{eq:gmmm-response}, together with the stated restrictions on the coefficients, gives the corresponding cases of \eqref{eq:general-mmm}.
\end{proof}

GMMM adds the construction of $E^G$ and $E^P$ and the definition of their values under each treatment sequence. Once these quantities are fixed, identification of the response coefficients is a rank problem.

\begin{proposition}[Identification of response coefficients and linear effects]
\label{prop:response-rank}
Fix the sequences of media inputs and all parameters of the media transformations. Suppose that the response model has an identity link and $b(W;\gamma)=W\gamma$. Let $H_0$ be the matrix whose columns are the transformed inputs for established media and let $W_*=(W,H_0)$. Let $X_\eta$ contain the variable for source state, $H^G$, $H^P$, and $H^GH^P$. If $M_{W_*}$ is the orthogonal projection onto the complement of the column space of $W_*$, define
\begin{align}
D_\eta
&=M_{W_*}X_\eta,
&
G_\eta
&=X_\eta^{\mathsf T}M_{W_*}X_\eta.
\label{eq:response-rank}
\end{align}
Conditional on the fixed sequences and transformations, the four response coefficients are identified from the conditional mean if and only if $G_\eta$ is nonsingular. Under a Gaussian conditional response with nonsingular variance, the same condition is equivalent to identification by the likelihood. For a fixed vector $d\in\R^4$, the linear effect $d^{\mathsf T}\beta$ is identified if and only if $d^{\mathsf T}v=0$ for every $v\in\R^4$ such that $D_\eta v=0$.
\end{proposition}
\begin{proof}
Two coefficient vectors $\beta$ and $\beta'$ give the same conditional mean after the nuisance coefficients are adjusted if and only if $X_\eta(\beta-\beta')$ belongs to the column space of $W_*$. Equivalently, $D_\eta(\beta-\beta')=0$. The four coefficients are identified exactly when $D_\eta$ has full column rank, which is equivalent to nonsingularity of $G_\eta$. The linear effect is constant over observationally equivalent coefficient vectors exactly when it vanishes on the null space of $D_\eta$. The necessity statement remains valid under the sign restrictions used in the simulations because their parameter space has nonempty interior.
\end{proof}

The simulations omit established media, so $H_0$ has no columns and residualization with respect to $W$ is sufficient. In the general model, $H_0$ must be included. For example, if a transformed input for an established channel equals $H^G$, its coefficient cannot be separated from $\beta_G$ even when the four columns of $X_\eta$ are linearly independent after removing $W$ alone.

A linear effect may remain identified when its component coefficients are not. This distinction is especially relevant when the variable for source state and $H^G$ are nearly proportional after the other regressors have been removed. Proposition~\ref{prop:response-rank} states the exact null-space condition; finite-sample stability still depends on the degree of collinearity.

The proposition treats the transformations as fixed. Joint identification of $(\beta_G,\alpha_G,\theta_G,\kappa_G)$ additionally requires the observed input sequences to provide independent information about response amplitude, carryover, and the Hill curve. With all other parameters fixed, full column rank of the Jacobian of the residualized mean with respect to these four parameters is sufficient for local identification at an interior point of a continuously differentiable model. If other parameters are unknown, full column rank of the Jacobian with respect to the complete free parameter vector after removal of the fixed linear controls is sufficient. A step treatment whose untransformed GEO input has one level before treatment and another after treatment cannot separate response amplitude from the Hill parameters using steady-state observations alone. Identification then depends on transition dynamics and on variation within source states, including variation in question counts, system shares, or occurrence probabilities. A prior can stabilize estimation without identifying a likelihood that lacks this variation. Related ambiguities arise when nonlinear response and coefficients that change over time produce similar observed sequences but different allocation decisions \citep{Dew2024yourmmm}.

\subsection{Occurrence Probabilities and Market Counts}
The market input in \eqref{eq:geo-exposure} weights an occurrence probability by the number and composition of questions, the shares of use across systems, and the probability of notice. In a special case, this difference amounts only to a change of scale. Let $A(\pi)$ denote linear carryover applied to a sequence of exact occurrence probabilities, with the same initialization as \eqref{eq:adstock}.

\begin{proposition}[Constant rescaling of the GEO input]
\label{prop:exposure-scale}
Suppose that $E^G_{it}(z)=c_0\pi_{it}(z)$ for the same constant $c_0>0$ in every market, period, and evaluated source state, including values used to initialize carryover. For the Hill function $h(a;\theta,\kappa)=a^\kappa/(a^\kappa+\theta^\kappa)$, it holds that
\begin{align}
h(A(E^G);\theta,\kappa)
&=
h(c_0A(\pi);\theta,\kappa)
=
h(A(\pi);\theta/c_0,\kappa).
\label{eq:constant-scale-equivalence}
\end{align}
Rescaling $\theta$ preserves the transformed sequence and the treatment effect computed from it. A single common value of $\theta$ does not generally absorb scale factors that vary across markets or periods.
\end{proposition}
\begin{proof}
Linearity gives $A(c_0\pi)=c_0A(\pi)$. Dividing the numerator and denominator of $h(c_0a;\theta,\kappa)$ by $c_0^\kappa$ proves \eqref{eq:constant-scale-equivalence}. For the final statement, consider two markets with the same positive occurrence-probability sequence and different constants $c_1$ and $c_2$. The transformed market counts differ because $h$ is strictly increasing, whereas a common transformation of the identical probability sequences is the same. One common value of $\theta$ cannot represent both.
\end{proof}

The same equivalence holds within each market if its scale factor $c_i$ is constant over time and the model permits a market-specific parameter $\theta_i$, rescaled to $\theta_i/c_i$. A Bayesian analysis must transform the prior and its support consistently. Under the common-scale condition in Proposition~\ref{prop:exposure-scale}, a common multiplicative level in question volume or notice probability is absorbed by the Hill midpoint. Question volume, system shares, and notice probabilities affect the treatment effect when their relative values vary across markets, periods, systems, question clusters, or source states. When that variation changes the proportionality between occurrence probabilities and expected counts, one occurrence-probability sequence cannot represent every market input. A finite collection of answers adds a separate source of uncertainty about the probabilities. Appendix~\ref{sec:uncertainty-integration} examines both issues in a static linear model and under a nonlinear transformation.

\subsection{Identification across Treatment Sequences}
The rank condition concerns the response model after the inputs have been fixed. Identification of $\Delta_G$ also requires the observed data to determine the expected response under each complete treatment sequence. For every $i\in\{1,\ldots,n\}$ and $t\in\{1,\ldots,T\}$, let $\mathcal H^-_{it}$ contain the variables observed immediately before the period-$t$ GEO and GEM treatments. It includes earlier responses, prior media inputs, earlier treatments, and any aggregate variables used to represent interference. Let $\mathcal H_{it}$ add the period-$t$ treatments and the values of $E^G_{it}$ and $E^P_{it}$ implied by them. For a pair $(a_G,a_P)$ of complete treatment sequences, let $F_{it}^{a_G,a_P}$ be the distribution of $\mathcal H_{it}$ obtained by recursively replacing the observed treatment assignment with those sequences. Finally, define $m_{it}(h)=\E[Y_{it}\mid\mathcal H_{it}=h]$. The following assumptions place the standard longitudinal g-formula on these GMMM treatment sequences.

\begin{assumption}[Consistency]
If the observed GEO and GEM treatments through period $t$ equal the values specified by $(a_G,a_P)$, then the observed variables through period $t$ equal their potential values under those treatment sequences.
\end{assumption}

\begin{assumption}[Sequential exchangeability]
Conditional on $\mathcal H^-_{it}$, the period-$t$ GEO and GEM treatments are independent of future potential information sets and potential responses under the treatment sequences evaluated in \eqref{eq:business-value}.
\end{assumption}

\begin{assumption}[Positivity]
For every value of $\mathcal H^-_{it}$ in the target support, each discrete treatment specified by the evaluated sequences has positive conditional probability. A positive continuous GEM spending level lies in the interior of its conditional support and has positive density in a neighborhood of that level. A sequence that sets spending to zero requires positive conditional mass at zero when the absence of a campaign is represented by a separate point mass.
\end{assumption}

\begin{assumption}[Observed-data laws in the target population]
\label{ass:measurement-population}
The measurement and response data identify $m_{it}$ and the conditional laws used to construct $F_{it}^{a_G,a_P}$ on the support of the evaluated treatment sequences. When the response model uses the expected counts in \eqref{eq:geo-exposure} and \eqref{eq:paid-exposure}, the measurement parameters and market inputs identify those counts for the target population. Identification of $m_{it}$ on the same support must also be established. If $\mathcal H_{it}$ instead contains unobserved realized impressions, their joint law with the observed variables and $Y_{it}$ must also be identified; a marginal distribution of impressions is insufficient.
\end{assumption}

\begin{assumption}[Interference through specified aggregates]
Spillovers within an observational unit are included in that unit's treatments and response. Treatments assigned to other units may affect its potential response or generative inputs only through aggregate variables included in $\mathcal H^-_{it}$. The distribution $F_{it}^{a_G,a_P}$ specifies how those aggregates evolve under the evaluated treatment sequences.
\end{assumption}

Under these conditions, the longitudinal g-formula applies to the complete GEO and GEM treatment sequences.

\begin{theorem}[Longitudinal g-formula for GMMM]
\label{thm:g-formula}
Under the preceding assumptions, the expected response in \eqref{eq:business-value} is identified by
\begin{align}
V(a_G,a_P)
&=
\sum_{i=1}^{n}
\sum_{t\in\calT_E}
\int
m_{it}(h)
\,dF_{it}^{a_G,a_P}(h).
\label{eq:g-formula}
\end{align}
The difference between this expression for $(a_G,a_P)=(1,1)$ and $(a_G,a_P)=(0,1)$ identifies $\Delta_G$.
\end{theorem}
\begin{proof}
Consistency links observed variables to their potential values under the realized treatments. Sequential exchangeability permits the observed assignment mechanism to be replaced, conditional on $\mathcal H^-_{it}$, by the evaluated treatment sequences. Positivity ensures that the required conditional laws are defined on their support. The measurement assumption identifies the conditional laws and $m_{it}$, including their dependence on the two constructed inputs. The interference assumption makes the potential response for each unit well defined at the chosen aggregation. Iterated expectation yields the longitudinal g-formula in \eqref{eq:g-formula} \citep{Robins1987agraphical}.
\end{proof}

Assumption~\ref{ass:measurement-population} states the general observed-data requirement. The next corollary gives concrete sufficient conditions for the parametric GMMM used in the simulations.

\begin{corollary}[Parametric identification with fixed market inputs]
\label{cor:expected-exposure}
Consider \eqref{eq:business-value} conditional on measured market inputs, with every response input other than GEO fixed across the two GEO treatment sequences. Suppose that consistency, sequential exchangeability, positivity, and the interference assumption hold. Assume that $\eta_M$ is identified on the required support and that the occurrence probabilities apply to the target population. Suppose that the media transformations are fixed or identified and that the response model with an identity link is correctly specified. If $W_*$ has full column rank and $D_\eta$ in Proposition~\ref{prop:response-rank} has rank four, then $\Delta_G$ is identified. More generally, for a known vector $d_\eta$, the condition $d_\eta^{\mathsf T}v=0$ whenever $D_\eta v=0$ is sufficient to identify $d_\eta^{\mathsf T}\beta$ even when the four coefficients are not separately identified.
\end{corollary}
\begin{proof}
The identified measurement parameters and fixed market inputs determine $E^G$ and $E^P$ under the two treatment sequences. Fixed or identified transformations then determine the corresponding regressors and their differences. Proposition~\ref{prop:response-rank} identifies either the four coefficients or the specified linear effect. The causal assumptions permit these identified quantities to be evaluated under both treatment sequences. Because the response model is written in terms of expected counts, these identified quantities suffice without modeling unobserved realized impressions.
\end{proof}

The corollary relies on occurrence probabilities that apply to the target population and on a correctly specified conditional response mean. Under a nonlinear transformation, a model for an expected count differs from a model for an unobserved realized count: for a random input $M$, $\E[h(M)\mid\calD_M]$ generally differs from $h(\E[M\mid\calD_M])$. Either specification also requires control of confounding between the media input and the response.

The general formula permits treatments to change variables observed later. The simulations condition on generated question counts, shares of system use, notice probabilities, and all response inputs other than GEO. For every parameter sample, the calculation replaces the complete GEO sequence before evaluating the response. The rollout dates are fixed by design, so the simulation obtains counterfactual values by evaluating the specified parametric response model under both sequences rather than by using nonparametric overlap for every cluster history.

A randomized rollout can supply treatment variation by design. An observational rollout instead requires pre-treatment controls and support for both treatment sequences. A variable such as referral traffic or brand search may transmit part of the effect of a current treatment to revenue or conversion. Fixing that variable does not recover the total effect and is not sufficient by itself to identify a controlled direct effect \citep{Acharya2016explainingcausal}. If the variable is affected by an earlier treatment and also influences a later treatment, it may belong to $\mathcal H^-_{it}$; the longitudinal g-formula then averages over its distribution under the treatment sequence \citep{Robins1987agraphical}.

\subsection{Decomposition under the Linear Response Model}
Under the linear response used in the simulations, $\Delta_G$ can be written as a component associated with the source state and a component due to the change in the transformed GEO input.

\begin{proposition}[Decomposition under the linear response model]
\label{prop:decomposition}
Suppose that $g_Y$ is the identity link and that every input to the response model other than the GEO source state and the GEO input is the same under the two treatment sequences. Assume that the terms below are integrable. Let $\boldsymbol Z$ denote the sequence of GEO source states selected by $a_G=1$, and let $\boldsymbol 0$ denote the sequence selected by $a_G=0$. Under \eqref{eq:gmmm-response}, it holds that
\begin{align}
\Delta_G
={}&
\beta_D\sum_{i=1}^{n}\sum_{t\in\calT_E}\E[Z_{it}]
+
\beta_G\sum_{i=1}^{n}\sum_{t\in\calT_E}
\E\left[H^G_{it}(\boldsymbol Z)-H^G_{it}(\boldsymbol 0)\right]
\nonumber\\
&+
\beta_{GP}\sum_{i=1}^{n}\sum_{t\in\calT_E}
\E\left[H^P_{it}
\left(H^G_{it}(\boldsymbol Z)-H^G_{it}(\boldsymbol 0)\right)\right].
\label{eq:decomposition}
\end{align}
\end{proposition}
\begin{proof}
Subtract the two conditional means. Every term that contains neither the GEO source state nor the GEO input cancels. The remaining terms are the term associated with the source state, the change in $H^G$, and the interaction between that change and $H^P$. Taking expectations and summing over $\calT_E$ gives \eqref{eq:decomposition}.
\end{proof}

The component due to the GEO input contains $(\beta_G+\beta_{GP}H^P_{it})(H^G_{it}(\boldsymbol Z)-H^G_{it}(\boldsymbol 0))$, so the interaction makes this component depend on the GEM input. Any other regressor changed by the GEO treatment would add another term.

Equation~\eqref{eq:decomposition} is an algebraic decomposition of the total GEO effect under the specified response model. We use it to study estimation error in the term for the source state and in the terms that contain the constructed GEO input. A causal mediation interpretation would require interventions that can vary the source state and the generated exposure separately, together with exchangeability and support for both interventions \citep{Imai2010ageneral,Pearl1995causaldiagrams}. Our estimand remains the causal effect of the complete source modification.

\section{Answer Collection and Simulation Evidence}
\label{sec:experiments}
The answer collection and the simulations serve different roles. The collection estimates occurrence probabilities under one observed source state. The simulations introduce changes in source state and generate the corresponding market responses, which permits evaluation of estimators of $\Delta_G$. Appendices~\ref{sec:response-collection-details} and~\ref{sec:additional-simulation-results} give the collection procedure and additional numerical results.

\subsection{Target Name in Generated Answers}
\label{sec:answer-collection}
We submit a fixed list of 56 questions that ask for product recommendations in seven use cases to GPT-5.6 Luna and GPT-4o. The list contains 28 English questions and 28 Japanese questions. Each API call uses the same system instruction and requires web search. Its user message contains only the requested language and question, and Glasp appears in neither message. After storing the complete answer, we record one if it contains a spelling of Glasp in a fixed dictionary and zero otherwise; the indicator records occurrence regardless of sentiment.

Each model answers every question 20 times, giving 2,240 answers. The requests are randomized within four blocks collected during one window, and Table~\ref{tab:model-occurrence} summarizes the resulting occurrence rates. Because the source state is unchanged throughout collection, these observations estimate probabilities for that state. Identification of the GEO treatment effect also uses variation in source state, market counts of the questions, shares of use across systems, notice probabilities, and contemporaneous response data.
\begin{table}[!htbp]
\centering
\caption{Occurrence of Glasp in the Collected Answers}
\label{tab:model-occurrence}
\footnotesize
\setlength{\tabcolsep}{3.8pt}
\begin{tabular}{lrrrcrr}
\toprule
Model & English & Japanese & Panel Rate & Posterior Mean (95\% Interval) & Mean Rank & Rank 1 (\%) \\
\midrule
GPT-4o & 35.4 & 20.2 & 27.8 & $28.8\ [27.0,30.7]$ & 1.62 & 60.5 \\
GPT-5.6 Luna & 28.6 & 38.9 & 33.8 & $34.5\ [32.8,36.3]$ & 1.94 & 39.4 \\
\bottomrule
\end{tabular}
\medskip
\begin{minipage}{0.94\textwidth}
\footnotesize
English, Japanese, and Panel Rate are percentages of answers that contain Glasp. Panel Rate averages the 56 observed rates with equal weights. Posterior means and intervals average the distributions for the 56 questions in \eqref{eq:occurrence-jeffreys}. Mean Rank is the rank at the first occurrence of Glasp among 11 candidate brands, conditional on its occurrence.
\end{minipage}
\end{table}

Glasp occurs in $33.8\%$ of the GPT-5.6 Luna answers and $27.8\%$ of the GPT-4o answers. Averaging the posterior distributions for the individual questions gives a difference of $5.70$ percentage points with a 95\% interval of $[3.12,8.27]$. The ordering reverses for English questions, where GPT-4o has the higher rate. Conditional on occurrence, GPT-4o also names Glasp first among the candidate brands more often. These results show why the recorded property and the weights assigned to the questions must match the business application. Appendix~\ref{sec:response-collection-details} reports results by language and use case.

\subsection{Simulation Design}
\label{sec:simulation-design}
The simulations compare estimators on the same panels and with the same construction of $E^G$ and $E^P$. Five geographic markets each contain six product clusters and are observed for 36 periods. Two clusters begin the GEO treatment in period 19 and two begin in period 23, while two never receive it. Every product cluster contains two question clusters observed on two generative systems. The resulting data contain 1,080 observations across units and periods but only six units of treatment assignment within each market. The response follows \eqref{eq:gmmm-response} with
\begin{align}
(\beta_D,\beta_G,\beta_P,\beta_{GP})
&=
(0.55,4.00,2.20,0.50).
\label{eq:true-coefficients}
\end{align}
In the controlled design, baseline occurrence probabilities follow the hierarchical logit model. A second design uses probabilities estimated from the collected answers and preserves the pairing of the two models for each question, as specified in \eqref{eq:paired-occurrence-probabilities}. The treatment effect of GEO is generated in both designs, and the evaluation window is $\calT_E=\{1,\ldots,36\}$. The simulations isolate uncertainty in the occurrence and notice models together with estimation of the media transformations: question counts and system shares are known, the demand variable is included among the controls, and every cell has 50 notice observations. A cell consists of one question cluster, one generative system, and one period. Established media channels are omitted. Appendix~\ref{sec:simulation-details} gives the complete data-generating process.

The principal comparisons use 120 paired replications in every condition. We evaluate the four response coefficients and $\Delta_G$ using root mean squared error (RMSE) and coverage of 95\% intervals. Appendix~\ref{sec:simulation-evaluation} defines the metrics and numerical settings. Paired intervals for differences between methods resample the 120 common replication indices 5,000 times. Differences between relative RMSE values are reported in percentage points. Appendix~\ref{sec:additional-simulation-results} gives additional results for one, five, and 20 answers per cell and for randomized estimates included in estimation.

\subsection{Estimation of Carryover and Saturation}
\label{sec:transformation-comparison}
Before comparing how uncertainty passes between the two models, we examine whether carryover and saturation are estimated or fixed. On the same 120 controlled panels with five answers per cell, one plug-in method fixes the carryover and Hill parameters at their prior means, while a second estimates them from the response data using the same candidates as the joint posterior. Three diagnostic benchmarks use the true transformation parameters. Table~\ref{tab:transformation-comparison} reports the results.
\begin{table}[!htbp]
\centering
\caption{Effect of Estimating the Transformation Parameters With Five Answers per Cell}
\label{tab:transformation-comparison}
\small
\begin{tabular}{llrrr}
\toprule
Estimator & Transformations & Effect RMSE (\%) & Vector RMSE & Coverage \\
\midrule
Plug-in & Prior mean & 17.642 & 1.207 & 0.975 \\
Plug-in & Estimated from response data & 17.525 & 0.945 & 0.975 \\
Two-stage & Prior distribution & 17.705 & 1.032 & 0.958 \\
Joint & Estimated from response data & 17.541 & 0.950 & 0.975 \\
Plug-in & Known & 17.135 & 0.882 & 0.950 \\
Two-stage & Known & 17.066 & 0.876 & 0.950 \\
Joint & Known & 17.191 & 0.884 & 0.958 \\
\bottomrule
\end{tabular}
\end{table}

Estimating the transformation parameters reduces the RMSE of the coefficient vector from $1.207$ to $0.945$ for the plug-in method. The corresponding value for the joint posterior is $0.950$. The 95\% paired interval for the joint value minus the plug-in value is $[-0.006,0.016]$, and the interval for the difference in effect RMSE also contains zero. Estimation of carryover and saturation explains the principal improvement in coefficient recovery over the plug-in method with fixed transformations. The designs based on the collected answers and the designs with misspecification give the same qualitative result (Appendix~\ref{sec:transformation-details}).

\subsection{Use of Response Data in the Measurement Model}
\label{sec:modular-comparison}
The cut posterior in \eqref{eq:cut-posterior} estimates the transformation parameters from the response data while keeping the distribution of $\eta_M$ fixed by $\calD_M$. Comparing it with the joint posterior isolates the effect of allowing $\calD_Y$ to revise $\eta_M$. We use $M=128$ samples of the measurement parameters, a common set of $J=700$ samples from the transformation prior, and 1,000 conditional samples of the response coefficients for each method. Appendix~\ref{sec:modular-diagnostics} describes the computation.

Let $J_A$ be the number of generated answers per cell. The controlled design combines $J_A\in\{1,5\}$ with standard deviations $\sigma_Y\in\{0.4,1.2\}$ for the response errors. A fifth condition uses occurrence probabilities estimated from the collected answers, with $J_A=5$ and $\sigma_Y=1.2$. The 120 replications in each condition give 600 panels. Methods use the same observations within a condition, and the controlled conditions use the same market inputs and standardized response innovations.
\begin{table}[!htbp]
\centering
\caption{Comparison on a Common Set of Transformation Draws}
\label{tab:modular-comparison}
\small
\begin{tabular}{lrrr}
\toprule
Method & Vector RMSE & Effect RMSE (\%) & Coverage \\
\midrule
\multicolumn{4}{l}{Controlled: $J_A=1$, $\sigma_Y=0.4$} \\
Plug-in (estimated) & 0.667 & 6.000 & 0.958 \\
Cut (estimated) & 0.687 & 6.010 & 0.950 \\
Joint & 0.657 & 6.046 & 0.950 \\
Two-stage (prior) & 0.918 & 7.075 & 0.975 \\
Oracle & 0.493 & 5.723 & 0.950 \\
\midrule
\multicolumn{4}{l}{Controlled: $J_A=1$, $\sigma_Y=1.2$} \\
Plug-in (estimated) & 0.975 & 17.450 & 0.967 \\
Cut (estimated) & 0.943 & 17.519 & 0.975 \\
Joint & 0.974 & 17.545 & 0.975 \\
Two-stage (prior) & 1.020 & 17.882 & 0.967 \\
Oracle & 0.886 & 17.236 & 0.950 \\
\midrule
\multicolumn{4}{l}{Controlled: $J_A=5$, $\sigma_Y=0.4$} \\
Plug-in (estimated) & 0.639 & 5.912 & 0.967 \\
Cut (estimated) & 0.642 & 5.930 & 0.967 \\
Joint & 0.636 & 5.923 & 0.975 \\
Two-stage (prior) & 0.897 & 7.070 & 0.958 \\
Oracle & 0.493 & 5.723 & 0.950 \\
\midrule
\multicolumn{4}{l}{Controlled: $J_A=5$, $\sigma_Y=1.2$} \\
Plug-in (estimated) & 0.956 & 17.513 & 0.975 \\
Cut (estimated) & 0.953 & 17.442 & 0.975 \\
Joint & 0.963 & 17.479 & 0.975 \\
Two-stage (prior) & 1.027 & 17.858 & 0.975 \\
Oracle & 0.886 & 17.236 & 0.950 \\
\midrule
\multicolumn{4}{l}{Estimated occurrence probabilities: $J_A=5$, $\sigma_Y=1.2$} \\
Plug-in (estimated) & 0.931 & 17.359 & 0.942 \\
Cut (estimated) & 0.912 & 17.326 & 0.942 \\
Joint & 0.939 & 17.480 & 0.950 \\
Two-stage (prior) & 0.974 & 18.099 & 0.950 \\
Oracle & 0.821 & 16.984 & 0.958 \\
\bottomrule
\end{tabular}
\end{table}
\noindent Each row uses 120 paired replications. Effect RMSE is the relative RMSE of $\Delta_G$. Coverage refers to its 95\% interval. The parenthetical labels indicate whether the transformation parameters are estimated from the response data or retained at their prior distribution. With 120 independent replications, coverage near $95\%$ has a binomial Monte Carlo standard error of about 2 percentage points.

A smaller standard deviation for the response errors substantially reduces effect RMSE, but the joint posterior has no uniform advantage over the cut posterior or the plug-in method that estimates the transformations. When $J_A=1$ and $\sigma_Y=0.4$, the joint posterior reduces coefficient RMSE relative to the cut posterior by $0.0300$; the 95\% paired interval for this reduction is $[0.0221,0.0383]$. The corresponding interval for the difference in effect RMSE contains zero. In the condition based on estimated occurrence probabilities, effect RMSE is $17.326\%$ for the cut posterior and $17.480\%$ for the joint posterior. Their difference is $0.154$ percentage points with an interval of $[0.046,0.259]$. The interval comparing the cut posterior with the plug-in method contains zero. Table~\ref{tab:modular-paired} reports all 15 paired comparisons.

To assess numerical sensitivity, we doubled both numbers of candidates for eight replications in each condition. The cut and joint estimates of $\Delta_G$ changed by at most $1.602\%$ of the true effect, although the effective sample size for the transformation candidates can approach one when the standard deviation of the response errors is small. Appendix~\ref{sec:modular-diagnostics} reports the full diagnostics and states which approximation errors remain when the candidate counts increase.

\subsection{Accuracy of the Components of the Treatment Effect}
\label{sec:effect-components}
The main simulation sets $\beta_D=0.55$ because a source modification can affect the response outside generated answers. The total effect therefore combines the term for the source state, the change in the transformed GEO input, and its interaction with GEM. Errors in the two components can partially cancel when the sum is estimated. Let $n_Z=\sum_{i=1}^{n}\sum_{t\in\calT_E}Z_{it}$ be the number of treated observations across units and periods. Define the component for the source state by $C_D=n_Z\beta_D$ and the component that contains the GEO input by $C_E=\Delta_G-C_D$. Every method estimates them as $\widehat C_D=n_Z\widehat\beta_D$ and $\widehat C_E=\widehat\Delta_G-\widehat C_D$, using its own estimate of $\beta_D$. The second component includes the interaction in Proposition~\ref{prop:decomposition} and is used here as a decomposition of the specified response model.

The treatment sequences give $n_Z=320$ and $C_D=176$ in every replication. With five answers per cell and $\sigma_Y=1.2$, the mean total effect is $266.944$ in the controlled design and $234.755$ in the design based on estimated occurrence probabilities. The corresponding mean values of $C_D/\Delta_G$ are $66.34\%$ and $75.38\%$. The term associated with the source state accounts for most of the simulated total effect even though its coefficient is unknown to the estimators.
\begin{table}[!htbp]
\centering
\caption{Accuracy of the Treatment Effect and Its Model Components}
\label{tab:effect-components}
\small
\begin{tabular}{lrrrr}
\toprule
Method & Total (\%) & Source State (\%) & GEO Input (\%) & Error Correlation \\
\midrule
\multicolumn{5}{l}{Controlled: $J_A=5$, $\sigma_Y=1.2$} \\
Plug-in (estimated) & 17.513 & 32.638 & 48.351 & -0.627 \\
Cut (estimated) & 17.442 & 32.526 & 47.756 & -0.626 \\
Joint & 17.479 & 32.646 & 48.366 & -0.631 \\
Two-stage (prior) & 17.858 & 36.426 & 47.860 & -0.677 \\
Oracle & 17.236 & 29.351 & 34.825 & -0.523 \\
\midrule
\multicolumn{5}{l}{Estimated occurrence probabilities: $J_A=5$, $\sigma_Y=1.2$} \\
Plug-in (estimated) & 17.359 & 27.551 & 47.129 & -0.555 \\
Cut (estimated) & 17.326 & 27.483 & 45.929 & -0.556 \\
Joint & 17.480 & 27.872 & 47.295 & -0.561 \\
Two-stage (prior) & 18.099 & 29.810 & 48.439 & -0.614 \\
Oracle & 16.984 & 24.470 & 31.468 & -0.409 \\
\bottomrule
\end{tabular}
\end{table}
\noindent The first three numerical columns report relative RMSE for each component. For a component $C$ and $R=120$ replications, the reported value is $100(R^{-1}\sum_{r=1}^{R}((\widehat C_r-C_r)/C_r)^2)^{1/2}$. The columns have different denominators. The final column is the correlation between estimation errors in the component associated with the source state and the component from the GEO input. Section~\ref{sec:modular-comparison} reports all five conditions.

In the condition based on estimated occurrence probabilities, the plug-in method has relative RMSE $17.359\%$ for the total effect and $47.129\%$ for the component from the GEO input. The latter value is $45.929\%$ for the cut posterior, $47.295\%$ for the joint posterior, and $31.468\%$ for the oracle. The negative correlations in Table~\ref{tab:effect-components} explain why the total can be more accurate than either component. If $e_D$ and $e_E$ denote their errors, the squared error of the total contains $2e_De_E$ in addition to the two squared component errors. Errors with opposite signs can cancel through this term. We retain $\Delta_G$ as the primary estimand and report the decomposition alongside the total effect.

\section{Discussion}
\label{sec:discussion}
GMMM separates construction of the two media inputs from the causal comparison. Generated answers determine occurrence probabilities; question counts, system shares, and notice probabilities put those probabilities on the scale of an MMM; complete treatment sequences then define the effects of GEO and GEM. Plug-in, cut, and joint procedures differ only in how they estimate the transformations and use uncertainty in the constructed inputs.

\subsection{Interpretation of the Results}
The GEO input describes expected noticed occurrence under a source state, and it may be positive both with and without the source modification. GEO changes this input through the occurrence probabilities. The estimand $\Delta_G$ compares the complete source-modification sequences, so it includes the path through the constructed input and the additional path represented by $\beta_DZ_{it}$. For GEM, the corresponding comparison changes the spending sequence that determines sponsored placements.

The collected answers also show that one overall occurrence rate is inadequate. The ordering of GPT-5.6 Luna and GPT-4o changes with language, and rates vary substantially across use cases. A feature used in an application should be chosen for the business question, and the probabilities for individual questions should receive the weights of the target population. Equal weights estimate the rate for the fixed panel used here. They also estimate a population rate when that population has the same distribution of questions.

Identification of $\beta_D$ and $\beta_G$ requires variation in the transformed GEO input beyond a proportional change in the variable for source state after the other regressors have been removed. When the two regressors move together, the condition on the null space in Proposition~\ref{prop:response-rank} determines whether $\Delta_G$ remains identified. The component results show why this distinction matters in finite samples: errors in the component associated with the source state and the component from the GEO input often have opposite signs and partially cancel in the total effect.

The design that assigns the treatment supplies its causal interpretation. A randomized rollout provides treatment variation by design. An observational rollout requires pre-treatment controls, support for both treatment sequences, and a contemporaneous comparison group when treated and untreated units experience common changes in the generative systems. Variables affected by an earlier treatment may also influence later assignment, in which case the g-formula in \eqref{eq:g-formula} averages over their distribution under each treatment sequence.

For estimation, the plug-in method that estimates carryover and saturation is a useful reference when the measurement distribution is concentrated. The cut posterior averages over measurement uncertainty while keeping the distribution of the measurement parameters fixed by the data used to construct the inputs. The joint posterior also uses the response likelihood to revise those parameters. We find no universal ranking: the joint posterior improves coefficient RMSE in one condition with a small standard deviation for the response errors, but it does not improve effect RMSE there and performs worse than the cut posterior in the condition based on estimated occurrence probabilities.

Numerical approximation raises a different question. Concentrated importance weights require checks of the number and placement of candidates, while a Markov chain would require checks of mixing and convergence. Increasing the number of candidates leaves the errors from the Laplace approximation for the measurement parameters and the Gaussian approximation to the posterior probability of the sign restrictions unchanged.

The simulations hold the GEO and GEM inputs fixed across estimators and omit established media channels, which isolates estimation of a common GMMM response model. A separate comparison of input construction would keep the response model and transformation estimation fixed while comparing a source indicator, occurrence probabilities, and expected counts under changes in question volume, system shares, and attention. A randomized estimate included in the likelihood informs the same treatment effect; an independently evaluated treatment would address predictive transfer to a different intervention.

\subsection{Alternative Response Models}
The construction of $E^G$ and $E^P$ precedes the choice of response distribution, which allows an application to use a link and distribution suited to sales, conversions, counts, or rates. Hierarchical coefficients can share information across related units, and coefficients that vary over time can represent changes in the relation between a media input and the response. Other carryover or saturation functions can replace \eqref{eq:adstock} and \eqref{eq:hill} without changing the definitions of the two expected counts or $\Delta_G$.

\subsection{Alternative Estimation Procedures}
The distinction between measurement and response parameters also applies outside the Bayesian formulation. A procedure based on likelihood can estimate all parameters together and use resampling for uncertainty, while a regularized procedure can penalize selected terms in \eqref{eq:general-criterion}. In the Bayesian analysis used here, the cut and joint posteriors differ in whether $\calD_Y$ updates $\eta_M$.

\subsection{Uncertainty in Market Inputs}
The simulations treat the question counts $N_{iqt}$ and shares $\omega_{ipt}$ as known. In an application, the counts may be estimated from a sample and the shares may change during the evaluation window. A probability model for either quantity can be added to $\eta_M$, after which its uncertainty passes through \eqref{eq:geo-exposure}, the media transformations, and the calculation of $\Delta_G$. This extension changes the construction of the input but not the response model or the definition of the treatment effect.

\section{Conclusion}
GMMM extends MMM to settings in which the media inputs associated with generative AI are not recorded directly. Under each source state, the GEO input is the expected number of generated answers in which a specified property occurs and is noticed; GEO changes this count but need not create it from zero. The GEM input is the expected number of sponsored placements that users notice. Both inputs are defined for each market and period before carryover and the Hill transformation are applied. The treatment effect of GEO is the total effect of the specified source modification under two complete treatment sequences.

The identification results allow the treatment effect to be determined in some designs even when its component coefficients cannot be recovered separately. They also identify the conditions under which a sequence of occurrence probabilities differs from a market count only by scale. In the simulations, estimating carryover and saturation explains most of the improvement in coefficient recovery over a plug-in method that fixes them. Allowing response data to revise the measurement parameters has no consistent benefit for estimating $\Delta_G$, and the total effect can conceal larger errors in its two model components.

The collected answers identify occurrence probabilities under the source state present during collection. Identification of the change caused by GEO uses observations under both source states together with contemporaneous market counts, notice observations, and business responses. GMMM combines these measurements and evaluates the causal effect from the complete sequence of media inputs.

\bibliography{arXiv2.bbl}

\bibliographystyle{tmlr}

\clearpage

\appendix

\section{Collection of Generated Answers and Simulation Inputs}
\label{sec:response-collection-details}
Section~\ref{sec:answer-collection} reports whether Glasp occurs in a fixed collection of generated answers. We specified the questions, model identifiers, language instructions, and search settings before collecting those answers.

\paragraph{Recorded indicator.}
An answer receives value one when it contains a spelling of Glasp in the target dictionary and zero otherwise. The dictionary is applied after the complete API record has been stored, so the target name is absent from the question and the instructions sent to the model. The indicator records occurrence regardless of whether the surrounding text endorses Glasp. The source state was constant during collection, so the answers identify occurrence probabilities only for that state.

\paragraph{Connection to the GEO input.}
The collected answers estimate occurrence probabilities under the information environment present during collection. Equation~\eqref{eq:geo-exposure} combines these probabilities with the number of relevant questions, the share handled by each generative system, and the probability of notice. Repeated calls characterize variation for each question and model, whereas the choice of questions determines the population described by their weighted average. Applying the estimated probabilities to a market requires weights for that market and generation settings that correspond to those used in the collection.

\subsection{Collection Design and Target Quantities}
The 56 questions ask for product recommendations in seven use cases, including web highlighting and summarization with AI. Each use case contains four questions in English and four in Japanese. The questions exclude the target name and its listed aliases. Each model receives the product question, the requested language, and a common system instruction that requires an independent recommendation and a web search before answering. The API call supplies a tool for web search and selects it through \texttt{tool\_choice}. Appendix~\ref{sec:reproducibility} gives the complete instruction and API settings.

For each question, we obtained 20 answers from GPT-5.6 Luna and 20 from GPT-4o. Four randomized blocks contained five repetitions for each combination of question and model, giving
\begin{align}
56\times2\times4\times5
&=2{,}240
\label{eq:answer-count}
\end{align}
API records. Collection ran from 05:58 UTC on September 3, 2026, to 04:40 UTC on September 4, 2026. The four blocks randomize request order within this single collection interval and do not represent distinct observation dates.

Let $R^{\mathrm{text}}_{qpbr}$ denote the complete answer for question $q\in\{1,\ldots,56\}$, model $p\in\{1,2\}$, block $b\in\{1,\ldots,4\}$, and repetition $r\in\{1,\ldots,5\}$. We define
\begin{align}
I_{qpbr}
&=
\mathbf 1\left(
\text{a listed spelling of Glasp occurs in }R^{\mathrm{text}}_{qpbr}
\right).
\label{eq:target-occurrence-indicator}
\end{align}
For the observed collection period and source state, $I_{qpbr}$ realizes the Bernoulli variable in \eqref{eq:occurrence-binomial}. The complete answer and the position of the first occurrence among the 11 candidate brands remain available for analyses of placement.

Let $t^\ast$ denote the collection period and $z^\ast$ the prevailing source state. The occurrence probability for question $q$ and model $p$ is $p_{qp}=\pi_{qpt^\ast}(z^\ast)$. Conditional on $p_{qp}$, the 20 calls are modeled as Bernoulli variables with a common probability during the collection interval. Let $J=20$ and $S_{qp}=\sum_{b=1}^{4}\sum_{r=1}^{5}I_{qpbr}$. The estimator for one question and model is $\widehat p_{qp}=S_{qp}/J$. Given weights $\omega_q\geq0$ such that $\sum_{q=1}^{56}\omega_q=1$, define the index for model $p$ by
\begin{align}
P_p
&=
\sum_{q=1}^{56}\omega_qp_{qp}.
\label{eq:panel-occurrence-rate}
\end{align}
The primary analysis sets $\omega_q=1/56$. The resulting index describes the fixed panel of questions; it equals a market occurrence rate only when the weights reproduce the market distribution of questions.

For each question and model, uncertainty from 20 answers is represented by
\begin{align}
p_{qp}\mid S_{qp}
&\sim
\operatorname{Beta}\left(
S_{qp}+\frac{1}{2},
J-S_{qp}+\frac{1}{2}
\right).
\label{eq:occurrence-jeffreys}
\end{align}
Independent samples from these distributions are averaged with the weights $\omega_q$ to obtain intervals for $P_p$ and differences between models. These intervals condition on the Bernoulli model and the fixed question weights. They exclude uncertainty from replacing the panel by a different population of questions.

All 2,240 requests returned complete answers, and no request or response identifier was duplicated. Each answer contains at least one completed call to the web search tool, giving 3,140 calls in total. GPT-5.6 Luna returned the requested model identifier, whereas GPT-4o returned the dated identifier \texttt{gpt-4o-2024-08-06}.

\subsection{Occurrence Results}
Table~\ref{tab:model-occurrence} reports the occurrence rates. Glasp occurs in 378 of 1,120 GPT-5.6 Luna answers and 311 of 1,120 GPT-4o answers, which gives raw panel rates of $33.8\%$ and $27.8\%$. Averaging the beta distributions for the 56 questions gives posterior means of $34.5\%$ and $28.8\%$. The posterior mean of the difference between the models is $5.70$ percentage points, with a 95\% posterior interval of $[3.12,8.27]$ percentage points and posterior probability above $0.999$ of being positive.

The aggregate rates conceal a reversal by language. GPT-4o exceeds GPT-5.6 Luna by $6.79$ percentage points for the English questions, whereas GPT-5.6 Luna exceeds GPT-4o by $18.75$ percentage points for the Japanese questions. The corresponding 95\% posterior intervals are $[2.90,10.02]$ and $[14.11,21.57]$ percentage points. Conditional on occurrence, Glasp is the first brand from the candidate list in $60.5\%$ of the GPT-4o answers and $39.4\%$ of the GPT-5.6 Luna answers. GPT-5.6 Luna has the higher overall rate but the lower frequency of first placement.

Figure~\ref{fig:occurrence-category-rates} shows variation across use cases. Both models have their highest occurrence rates for questions about web highlighting and their lowest rates for questions about knowledge management. The reversal by language is especially pronounced for summarization with AI and YouTube learning. A model with system-specific effects for the questions can represent this variation, whereas a model with one common occurrence probability cannot.
\begin{figure}[!htbp]
\centering
\includegraphics[width=0.96\textwidth,draft=false]{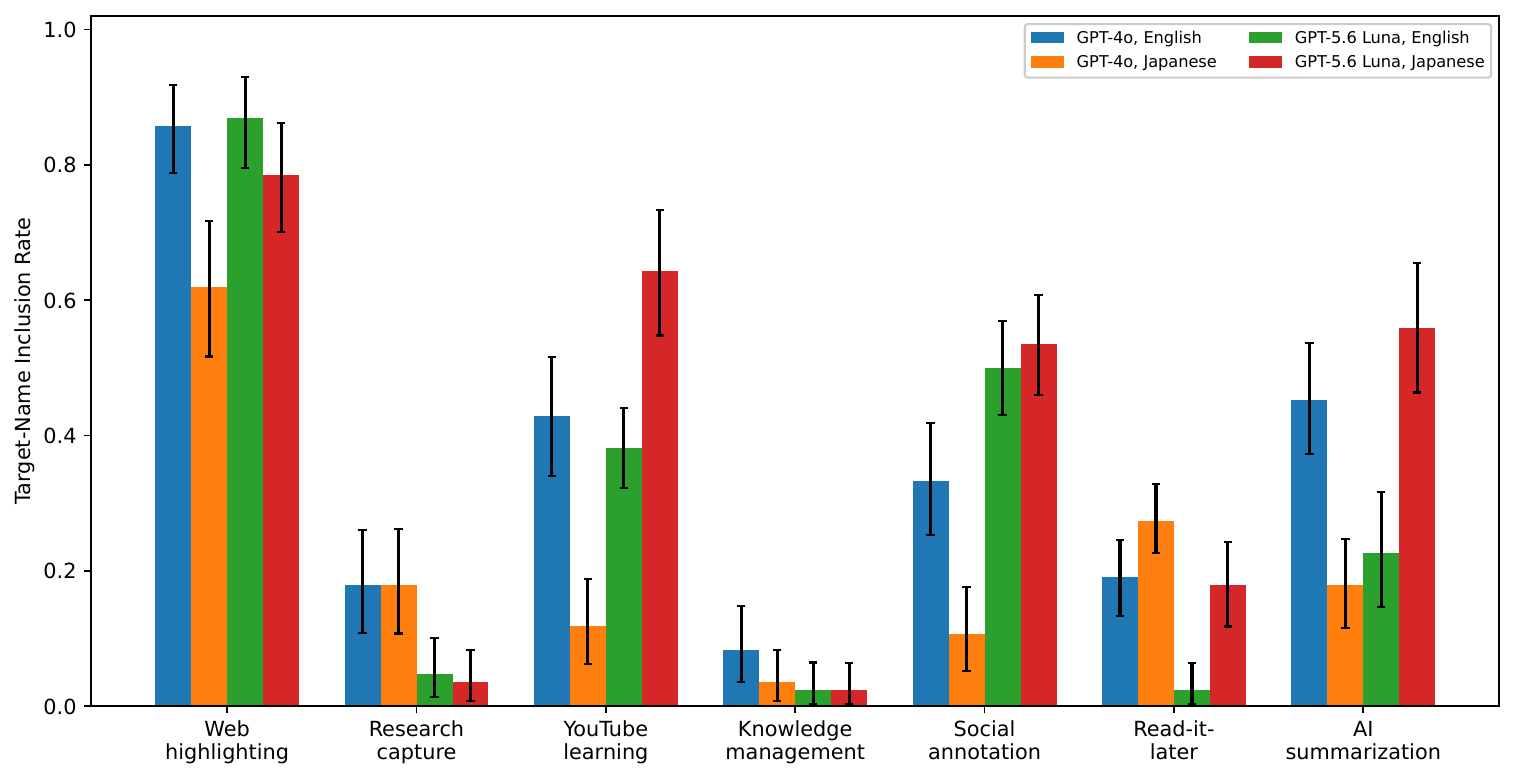}
\caption{Occurrence rates by use case and answer language. Bars show posterior means, and error bars show 95\% intervals obtained by averaging the distributions for the four questions in each cell.}
\label{fig:occurrence-category-rates}
\end{figure}
The ordering of the two models is the same in every collection block. GPT-4o rates range from $24.3\%$ to $30.0\%$, and GPT-5.6 Luna rates range from $31.4\%$ to $36.8\%$. Because the blocks cover different parts of one collection interval, comparisons across them assess stability during collection and do not estimate a time trend.

\subsection{Occurrence Probabilities Used in the Simulations}
The variation across questions supplies baseline probabilities for a second simulation design. Within each language, queries are sampled with replacement, and the same sampled query is used for both models. The treatment effect of GEO is generated by the simulation and does not equal the observed difference between the models.

Let $C_{jp}$ be the occurrence count for question $j$ and model $p$ among $J=20$ answers. In each replication, six English and six Japanese question indices are sampled with replacement within language. The selected index $s(q)$ is common to the two models. Conditional on that selection, the baseline probabilities are sampled independently across models from
\begin{align}
\pi^0_{qp}\mid s(q),\calD_M
&\sim
\operatorname{Beta}\left(
C_{s(q),p}+\tfrac12,
J-C_{s(q),p}+\tfrac12
\right).
\label{eq:paired-occurrence-probabilities}
\end{align}
The common question selection retains the observed pairing and language composition. Variation within each beta distribution represents uncertainty from the finite number of answers. The sampled probabilities determine the baseline log odds, after which the simulation adds the same time component and treatment shift in log odds as the controlled design. The simulation specifies the response coefficients and treatment shift independently of the collected answers.

The fitted occurrence model includes a system-specific centered question effect, permitting the difference in log odds between systems to vary across questions. The remaining models used to construct the inputs and the equation for the business response are unchanged.

\section{Additional Simulation Results}
\label{sec:additional-simulation-results}

The simulations keep the panel structure fixed while varying the number of generated answers, the treatment of the transformation parameters, and the relation between the constructed inputs and the business response. Because the treatment assignments are common across methods, the comparisons can isolate the source of each difference before examining numerical accuracy and alternative data-generating processes.

\subsection{Panel Design and Evaluation Criteria}
\label{sec:simulation-evaluation}

The controlled simulation observes five geographic markets for 36 periods, with $\calT_E=\{1,\ldots,36\}$. Each market contains six product clusters, every product cluster is linked to two question clusters, and every question cluster is observed on two generative systems. Two product clusters begin the GEO treatment in period 19 and two begin it in period 23; the remaining two never receive GEO. Occurrence of the target name follows the hierarchical logit model. The number of relevant questions varies with latent demand and seasonality, while system shares vary across markets. GEM spending depends on demand observed before treatment and on promotion, and a model for sponsored placements determines the number shown.

The design using estimated occurrence probabilities retains the same panel and treatment structure. Its baseline probabilities follow \eqref{eq:paired-occurrence-probabilities}, and the fitted measurement model includes interactions between question and system. The demand variable that affects the response is included among the controls in both simulations. The panel has 1,080 observations, while treatment is assigned at the level of six product clusters and shared across markets.

The response coefficients are given in \eqref{eq:true-coefficients}.

For each collection size of one, five, or 20 answers per cell, we run 120 paired replications. The market panel and the simulated randomized estimate are held fixed across collection sizes. Each fit uses 700 candidates from the proposal distribution and 1,000 samples from the conditional posterior for the response parameters defined in Section~\ref{sec:method}. The plug-in benchmark fixes the sequences of media inputs. The two-stage procedure gives equal weight to samples of the measurement parameters, whereas Joint GMMM reweights the same samples with the response data. The randomized estimate from the target population is the average treatment effect of GEO under the specified rollout relative to the complete sequence with GEO disabled. A second version adds $0.75$ to represent a difference between the experimental and target populations. The oracle observes the true sequences of media inputs and transformations.

We assess the response coefficients and $\Delta_G$ by root mean squared error (RMSE) and coverage of 95\% intervals. For $R$ replications, the RMSE of the coefficient vector is $(\sum_{r=1}^R\|\widehat\beta_r-\beta_0\|_2^2/(4R))^{1/2}$. Relative effect RMSE is $(R^{-1}\sum_{r=1}^R((\widehat\Delta_{G,r}-\Delta_{G,r})/\Delta_{G,r})^2)^{1/2}$, reported as a percentage. Paired bootstrap intervals use 5,000 resamples of the common replication indices and recompute both RMSEs. Differences between relative effect RMSEs are measured in percentage points. Effective sample size (ESS) describes concentration of the importance weights. Results with five answers per cell represent the central collection size, while the comparison that adds a randomized estimate uses 20 answers per cell to reduce uncertainty from the answer collection.

\subsection{Simulation Based on the Collected Answers}
\label{sec:estimated-occurrence-simulation}

Using the paired question probabilities in \eqref{eq:paired-occurrence-probabilities}, this design assigns GPT-5.6 Luna to the first system and GPT-4o to the second. Both models share the six selected query indices per language in each replication, while the fitted model with interactions between question and system allows their baseline log-odds differences to vary across queries.

With five simulated answers per cell, the RMSE of the coefficient vector is $0.942$ for Joint GMMM, $1.120$ for Plug-in GMMM, and $0.981$ for Two-stage GMMM (Table~\ref{tab:estimated-occurrence-simulation}). The paired difference for Joint minus Plug-in is $-0.178$, with a 95\% bootstrap interval of $[-0.261,-0.096]$. For Joint minus Two-stage, the difference is $-0.038$ and the interval is $[-0.096,0.021]$. The comparison with Plug-in GMMM supports lower coefficient error when the transformation parameters are estimated. For $\Delta_G$, the RMSEs are $17.576\%$, $18.316\%$, and $18.258\%$, and both paired intervals for Joint GMMM include zero. Section~\ref{sec:transformation-comparison} separates the role of estimating the transformation parameters from the role of averaging over measurement uncertainty.

\begin{table}[!htbp]
\centering
\caption{Simulation Based on Occurrence Probabilities Estimated From Collected Answers}
\label{tab:estimated-occurrence-simulation}
\footnotesize
\setlength{\tabcolsep}{3.8pt}
\begin{tabular}{lrrrrr}
\toprule
Method & Vector RMSE & Effect Bias (\%) & Effect RMSE (\%) & Coverage & Median ESS \\
\midrule
Plug-in GMMM & 1.120 & -0.966 & 18.316 & 0.950 & -- \\
Two-stage GMMM & 0.981 & -0.246 & 18.258 & 0.950 & -- \\
Joint GMMM & 0.942 & 0.380 & 17.576 & 0.958 & 91.2 \\
Joint + aligned estimate & 0.938 & 0.458 & 17.057 & 0.967 & 91.3 \\
Joint + shifted estimate & 0.942 & 3.813 & 17.804 & 0.950 & 91.4 \\
Oracle & 0.823 & 0.369 & 17.049 & 0.958 & -- \\
\bottomrule
\end{tabular}

\medskip
\begin{minipage}{0.94\textwidth}
\footnotesize
Effect Bias and Effect RMSE report the relative bias and relative RMSE of the treatment effect of GEO as percentages. Coverage is the empirical coverage of its 95\% interval. ESS is the effective sample size of the importance weights for the joint computations. ``Aligned estimate'' uses the randomized estimate from the target population, and ``shifted estimate'' adds $0.75$ to that estimate. Dashes mark methods whose weights are fixed by construction. Each row is based on 120 paired replications. A cell is defined by a question, a generative system, and a period.
\end{minipage}
\end{table}

Adding the randomized estimate from the target population lowers the RMSE of $\Delta_G$ from $17.576\%$ to $17.057\%$. The paired difference is $-0.519$ percentage points with a 95\% bootstrap interval of $[-0.862,-0.182]$. The RMSE of the coefficient vector changes only from $0.942$ to $0.938$, and the paired interval for that change includes zero. When $0.75$ is added to the randomized estimate, signed bias in $\Delta_G$ rises from $0.380\%$ to $3.813\%$. Its RMSE is $17.804\%$, and the interval for its paired difference from Joint GMMM is $[-0.387,0.826]$ percentage points.

Figure~\ref{fig:estimated-occurrence-rmse} traces coefficient error as the number of answers increases. More answers reduce uncertainty in the occurrence probabilities, while variation in the response and uncertainty in the transformation parameters remain. The fitted model allows the difference between systems to vary across questions. The controlled comparison below holds the baseline probability model fixed and isolates the effect of averaging over measurement uncertainty.

\begin{figure}[!htbp]
\centering
\includegraphics[width=0.82\textwidth,draft=false]{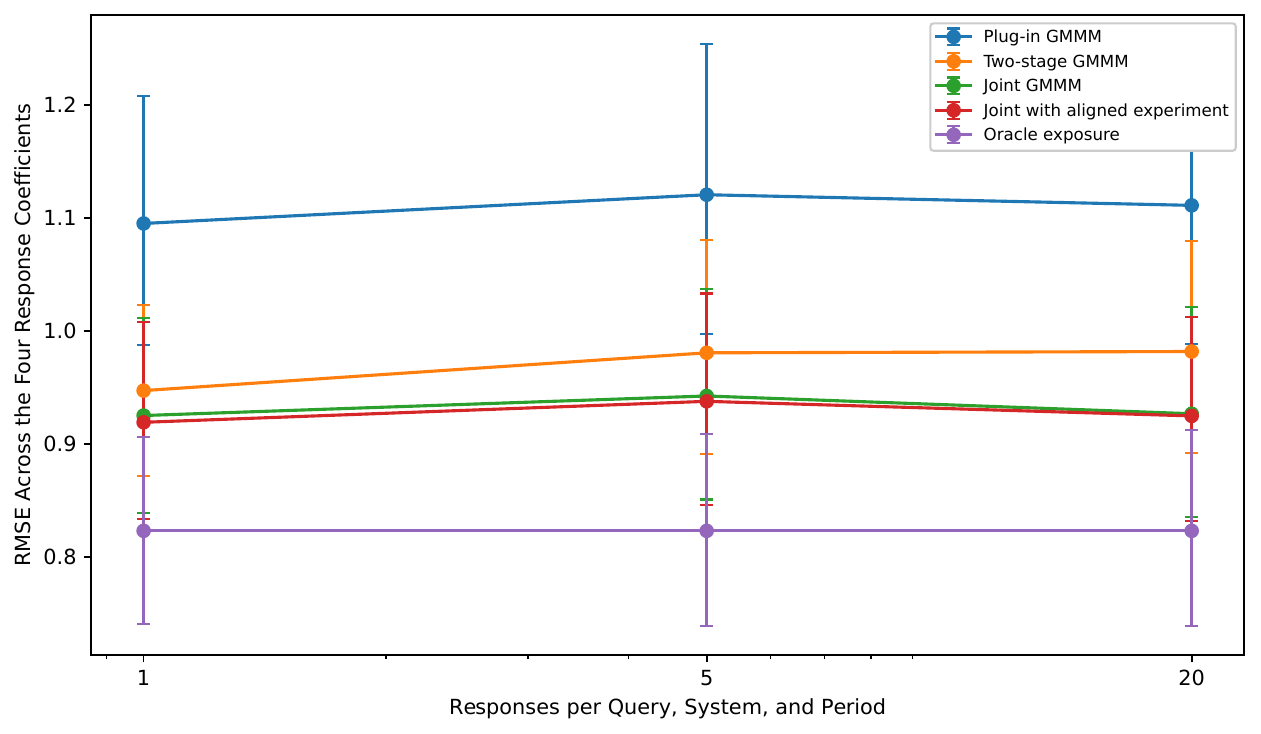}
\caption{Root mean squared error of the coefficient vector in the simulation using estimated occurrence probabilities. The baseline distributions of occurrence indicators are estimated from the collected generated answers. Lines connect ordered numbers of answers, and error bars are 95\% bootstrap intervals over replications.}
\label{fig:estimated-occurrence-rmse}
\end{figure}

\subsection{Controlled Simulation Results}

With five answers per cell, the RMSE of the coefficient vector is $0.950$ for Joint GMMM, $1.207$ for Plug-in GMMM, and $1.032$ for Two-stage GMMM (Table~\ref{tab:gmmm-main}). Joint GMMM also reduces the RMSE for the coefficient on the GEO input from $2.068$ to $1.558$ and for the coefficient on the GEM input from $0.754$ to $0.618$ relative to Plug-in GMMM. The oracle has vector RMSE $0.880$, so variation in the response remains important even when the sequences of media inputs and transformations are known.

\begin{table}[!htbp]
\centering
\caption{Root Mean Squared Error With Five Answers per Cell}
\label{tab:gmmm-main}
\footnotesize
\setlength{\tabcolsep}{3.7pt}
\begin{tabular}{lrrrrrr}
\toprule
Method & Source State & GEO & Sponsored & Interaction & Vector & Effect RMSE (\%) \\
\midrule
Plug-in GMMM & 0.216 & 2.068 & 0.754 & 0.965 & 1.207 & 17.642 \\
Two-stage GMMM & 0.201 & 1.704 & 0.722 & 0.893 & 1.032 & 17.705 \\
Joint GMMM & 0.183 & 1.558 & 0.618 & 0.876 & 0.950 & 17.541 \\
Joint + aligned estimate & 0.182 & 1.549 & 0.614 & 0.869 & 0.944 & 17.441 \\
Joint + shifted estimate & 0.186 & 1.559 & 0.622 & 0.884 & 0.953 & 17.932 \\
Oracle & 0.161 & 1.429 & 0.494 & 0.885 & 0.880 & 17.158 \\
\bottomrule
\end{tabular}
\medskip
\begin{minipage}{0.94\textwidth}
\footnotesize
``Aligned estimate'' uses the randomized estimate from the target population, and ``shifted estimate'' adds $0.75$ to that estimate.
\end{minipage}
\end{table}

Because every method uses the same simulated panel, the paired comparisons attribute differences to the estimators. Joint GMMM reduces the RMSE of the coefficient vector by $0.256$ relative to Plug-in GMMM and by $0.082$ relative to Two-stage GMMM. Both 95\% bootstrap intervals exclude zero. The relative RMSE of the treatment effect of GEO is similar across the three methods because errors in the individual coefficients can offset in the scalar estimand.

\begin{table}[!htbp]
\centering
\caption{Paired Root Mean Squared Error Differences With Five Answers per Cell}
\label{tab:paired-differences}
\small
\begin{tabular}{llrrr}
\toprule
Comparison & Metric & Difference & Lower & Upper \\
\midrule
Joint minus Plug-in GMMM & Vector RMSE & -0.2564 & -0.3322 & -0.1745 \\
Joint minus Plug-in GMMM & Effect RMSE (pp) & -0.1010 & -0.7087 & 0.4884 \\
Joint minus Two-stage GMMM & Vector RMSE & -0.0821 & -0.1403 & -0.0215 \\
Joint minus Two-stage GMMM & Effect RMSE (pp) & -0.1645 & -0.7435 & 0.3804 \\
\bottomrule
\end{tabular}
\end{table}

Figures~\ref{fig:coefficient-rmse} and~\ref{fig:effect-rmse} show the same pattern across numbers of answers: joint estimation has lower coefficient error than the plug-in estimator with fixed transformations, while their errors in the treatment effect of GEO remain close. This comparison changes both the treatment of measurement uncertainty and the estimation of the transformation parameters. For that reason, the comparison in Section~\ref{sec:transformation-comparison} varies these features one at a time.

\begin{figure}[!htbp]
\centering
\includegraphics[width=0.82\textwidth,draft=false]{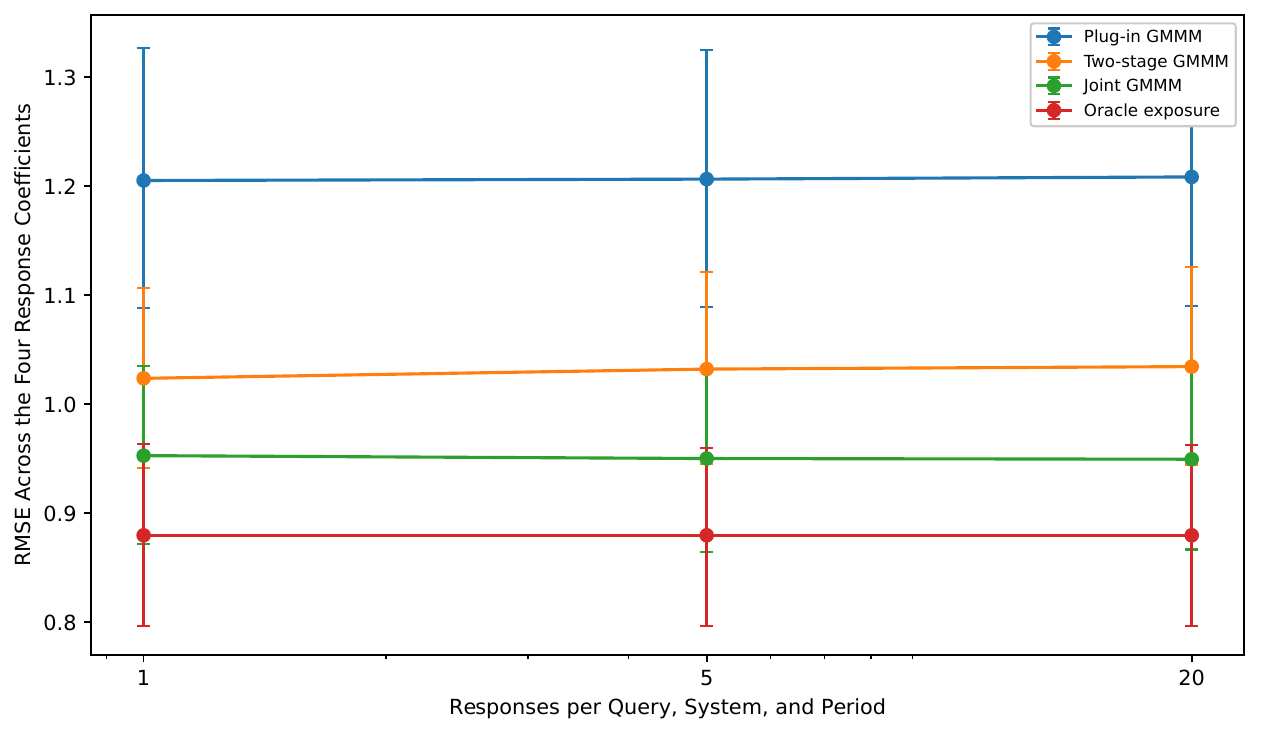}
\caption{Root mean squared error of the coefficient vector by answers per cell. Error bars are 95\% bootstrap intervals over replications.}
\label{fig:coefficient-rmse}
\end{figure}

\begin{figure}[!htbp]
\centering
\includegraphics[width=0.82\textwidth,draft=false]{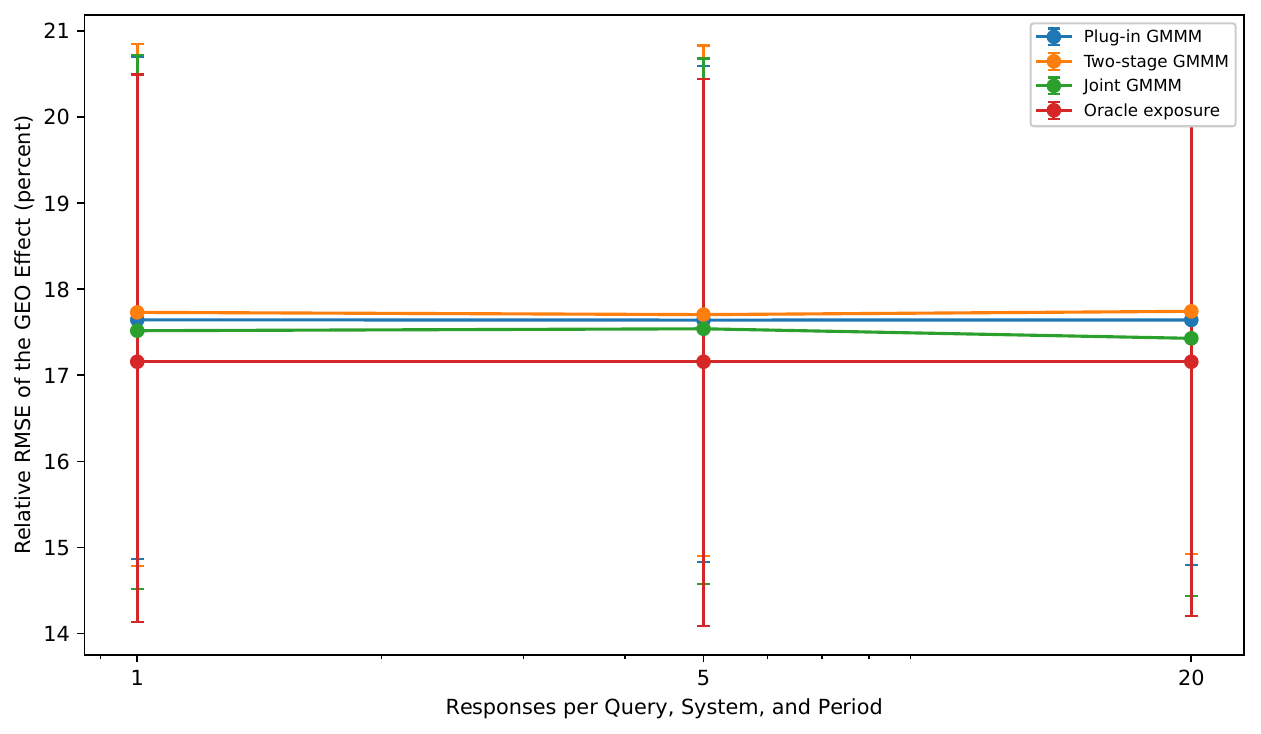}
\caption{Relative root mean squared error for the treatment effect of GEO by answers per cell. Error bars are 95\% bootstrap intervals over replications.}
\label{fig:effect-rmse}
\end{figure}

\subsection{Role of Estimating Carryover and Saturation}
\label{sec:transformation-details}

We use the same 120 controlled panels with five answers per cell to isolate the contribution of estimating the transformations. In the first comparison, the plug-in estimator uses the response data to estimate the parameters of the media transformations from the same candidates as Joint GMMM while its measurement parameters remain fixed. The second comparison supplies the true transformations to the plug-in, two-stage, and joint estimators, leaving only their treatment of measurement uncertainty to differ. This second comparison is diagnostic because the true transformations would be unavailable in an application. All comparisons retain 700 importance candidates and 1,000 posterior samples.

Estimating the transformation parameters while fixing the measurement inputs lowers vector RMSE from $1.207$ to $0.945$. The paired difference is $-0.262$ with a 95\% interval of $[-0.339,-0.182]$. Joint GMMM has vector RMSE $0.950$, only $0.005$ above the plug-in estimator with estimated transformations, and the interval for this difference is $[-0.006,0.016]$. Their relative RMSEs for $\Delta_G$ are $17.541\%$ and $17.525\%$, with a paired interval that contains zero. When the true transformation parameters are supplied, vector RMSE ranges from $0.876$ to $0.884$ and relative effect RMSE ranges from $17.066\%$ to $17.191\%$. The controlled design attributes most of the improvement over fixed transformations to their estimation. Averaging over measurement uncertainty provides no clear additional improvement in $\Delta_G$ in this design.

We also compare estimated transformations on 120 matched panels in each of three designs: the design using estimated occurrence probabilities, shifted transformations, and combined misspecification. Each panel uses five answers per cell, 700 importance candidates, and 1,000 samples from the conditional posterior for the response parameters. Plug-in and joint estimators receive the same transformation samples and panels generated with common random numbers, but the plug-in estimator fixes the measurement parameters at their fitted values. Table~\ref{tab:matched-transformations} reports the three methods needed to isolate the role of estimating the transformation parameters.

\begin{table}[!htbp]
\centering
\caption{Comparison After Estimating the Transformation Parameters}
\label{tab:matched-transformations}
\small
\begin{tabular}{llrrr}
\toprule
Design & Method & Vector RMSE & Effect RMSE (\%) & Coverage \\
\midrule
estimated occurrence probabilities & Plug-in (fixed) & 1.120 & 18.316 & 0.950 \\
 & Plug-in (estimated) & 0.929 & 17.417 & 0.958 \\
 & Joint (estimated) & 0.942 & 17.576 & 0.958 \\
\addlinespace
Shifted transformations & Plug-in (fixed) & 1.815 & 23.131 & 0.808 \\
 & Plug-in (estimated) & 0.952 & 19.135 & 0.908 \\
 & Joint (estimated) & 0.961 & 19.133 & 0.900 \\
\addlinespace
Combined misspecification & Plug-in (fixed) & 2.014 & 22.738 & 0.783 \\
 & Plug-in (estimated) & 1.071 & 18.048 & 0.917 \\
 & Joint (estimated) & 1.080 & 18.059 & 0.917 \\
\bottomrule
\end{tabular}
\end{table}

\noindent Vector RMSE measures recovery of the four response coefficients. Effect RMSE is relative to the true treatment effect of GEO, and coverage refers to its 95\% interval. ``Fixed'' uses the transformations fixed at their prior means, whereas ``estimated'' allows the response data to change the weights assigned to transformation samples.

In the design using estimated occurrence probabilities, vector RMSE is $0.929$ for the plug-in estimator with estimated transformations and $0.942$ for Joint GMMM. The paired difference for Joint minus Plug-in is $0.0133$, with a 95\% bootstrap interval of $[0.0008,0.0262]$. The corresponding relative RMSEs for $\Delta_G$ are $17.417\%$ and $17.576\%$; their difference is $0.159$ percentage points with an interval of $[-0.012,0.330]$. Estimating the transformations also brings the plug-in and joint results close under shifted transformations and combined misspecification. Table~\ref{tab:matched-paired} shows that all three intervals for the difference in effect RMSE contain zero. Across these designs, the improvement over the plug-in method with fixed transformations comes chiefly from estimating the transformations.

\subsection{Cut and Joint Posteriors With Common Candidates}
\label{sec:modular-results-details}

Because the two-stage benchmark retains the transformation prior, it cannot isolate the effect of allowing response data to revise the measurement parameters. We make that comparison by evaluating the cut posterior in \eqref{eq:cut-posterior} and the joint posterior on the same $M\times J$ combinations of measurement and transformation parameters. We also retain a plug-in estimator that estimates the transformations, the two-stage benchmark, and an oracle that observes the true sequences of media inputs and the true transformations. The computation crosses $M=128$ samples of the measurement parameters with $J=700$ common transformation samples and then takes 1,000 conditional coefficient samples for each method. This candidate set differs from the 700 independent pairs of measurement and transformation parameters used in the preceding tables, so the numerical levels can differ even when the market panels are the same.

Let $J_A$ denote the number of answers per cell. The controlled design crosses $J_A\in\{1,5\}$ with standard deviations $\sigma_Y\in\{0.4,1.2\}$ for the response errors. A fifth condition uses estimated occurrence probabilities with $J_A=5$ and $\sigma_Y=1.2$. With 120 replications per condition, this gives 600 primary panels. Controlled conditions share market inputs and standardized response innovations, and all methods within a condition use the same observations. Query volumes and shares of system use remain known, with 50 notice observations per cell. Holding those sources of information and the construction of $E^G$ and $E^P$ fixed lets us compare the estimators as the occurrence probabilities become more precise and the standard deviation of the response errors changes. Comparing alternative media inputs would instead require changing how $E^G$ and $E^P$ are constructed.

Table~\ref{tab:modular-comparison} reports every condition. When $\sigma_Y$ falls from $1.2$ to $0.4$, the oracle's relative effect RMSE falls from $17.236\%$ to $5.723\%$. Methods that estimate the transformation parameters have RMSE near $6\%$, while the two-stage benchmark remains near $7\%$. The concentration near $17\%$ in the original design reflects the standard deviation of the response errors. Even when that standard deviation is smaller, Joint GMMM does not uniformly improve estimation of $\Delta_G$.

\noindent Each row uses 120 paired replications. Effect RMSE is the relative RMSE of $\Delta_G$, expressed as a percentage, and Coverage is the empirical coverage of its 95\% interval. Cut and Joint estimate the transformation parameters at each measurement sample, while Plug-in estimates them at the measurement point estimate. Two-stage retains the transformation prior. With 120 independent replications, coverage near $95\%$ has a binomial Monte Carlo standard error of about 2 percentage points.

With one answer per cell and $\sigma_Y=0.4$, Joint GMMM lowers vector RMSE relative to the cut posterior by $0.0300$; the 95\% paired bootstrap interval for this reduction is $[0.0221,0.0383]$. The difference in relative effect RMSE is $0.035$ percentage points with an interval of $[-0.053,0.131]$. When $\sigma_Y=1.2$, the cut posterior has lower vector RMSE at both collection sizes. Allowing the response data to revise the measurement parameters can affect coefficient recovery without improving estimation of $\Delta_G$.

In the condition using estimated occurrence probabilities, effect RMSE is $17.359\%$ for Plug-in, $17.326\%$ for Cut, and $17.480\%$ for Joint. The difference for Joint minus Cut is $0.154$ percentage points, with an interval of $[0.046,0.259]$, while the interval for Cut minus Plug-in includes zero. Across the five conditions, none of the paired comparisons of effect RMSE favors the joint posterior over the cut posterior or the plug-in estimator with estimated transformations. Table~\ref{tab:modular-paired} reports all 15 comparisons, including the cases in which the joint posterior improves coefficient recovery.

These comparisons approximate the specified posterior distributions with a finite set of candidates. Doubling both candidate counts on eight replications selected in advance per condition changes estimates from the cut and joint posteriors by at most $1.602\%$ of the true effect. The effective sample size within a measurement sample can approach one, especially when the standard deviation of the response errors is small. Appendix~\ref{sec:modular-diagnostics} reports the checks in full. The paired intervals describe these computations; they do not determine the ranking under exact posterior integration.

\subsection{Numerical Accuracy}

Increasing the number of importance candidates from 700 to 1,400 changes root mean squared error of the coefficient vector from $0.950$ to $0.947$ and root mean squared error of the treatment effect of GEO from $17.541\%$ to $17.525\%$. Median effective sample size rises from $124.3$ to $251.0$. Across replications, the absolute change in the coefficient on the GEO input has median $0.087$ and 90th percentile $0.253$, while the median absolute change in $\Delta_G$ is $0.386\%$ of its true value.

\begin{table}[!htbp]
\centering
\caption{Sensitivity to the Number of Importance Candidates}
\label{tab:importance-candidate-sensitivity}
\small
\begin{tabular}{rrrrrrr}
\toprule
Importance Candidates & GEO & GEM & Vector & Effect RMSE (\%) & ESS & ESS $q_{0.10}$ \\
\midrule
700 & 1.558 & 0.618 & 0.950 & 17.541 & 124.266 & 40.394 \\
1400 & 1.556 & 0.613 & 0.947 & 17.525 & 250.993 & 84.884 \\
\bottomrule
\end{tabular}
\end{table}

The lower tail of ESS in Table~\ref{tab:importance-diagnostics} shows why we report weight concentration alongside the estimates. In the correctly specified design, the tenth percentile rises from $40.4$ with 700 candidates to $84.9$ with 1,400 candidates, whereas it is $8.6$ under combined misspecification with 700 candidates. Panels with low ESS remain in the reported distributions. Increasing the candidate count assesses the stability of importance sampling, while the hierarchical Laplace approximation and the numerical probabilities used for coefficient constraints must be assessed by other calculations. We have not compared these computations with exact posterior integration.

\begin{table}[!htbp]
\centering
\caption{Diagnostics for the Importance Weights}
\label{tab:importance-diagnostics}
\small
\begin{tabular}{lrrrr}
\toprule
Design & Median & $q_{0.10}$ & $q_{0.05}$ & Minimum \\
\midrule
Correctly specified, K=700 & 124.266 & 40.394 & 15.465 & 3.042 \\
Correctly specified, K=1400 & 250.993 & 84.884 & 36.550 & 8.176 \\
Combined misspecification, K=700 & 25.947 & 8.616 & 5.466 & 2.779 \\
\bottomrule
\end{tabular}
\end{table}

\subsection{Alternative Data-Generating Processes}

The remaining designs alter either the data used to construct the inputs or the response equation. Two add overdispersion to the occurrence indicators or the counts of sponsored placements; the others shift the transformation parameters, add a nonlinear term omitted from the fitted response model, or combine these changes. Each design uses 120 replications with five answers per cell, 700 importance candidates, and 1,000 samples from the conditional posterior for the response parameters. Table~\ref{tab:gmmm-robustness} reports relative RMSE for $\Delta_G$, vector RMSE, interval coverage, and weight concentration.

Adding overdispersion to the occurrence indicators or the counts of sponsored placements produces similar relative effect RMSE across the three estimators. Under shifted transformations, relative effect RMSE is $19.133\%$ for Joint GMMM, $23.131\%$ for Plug-in GMMM, and $22.550\%$ for Two-stage GMMM. When all departures are combined, the corresponding value for Joint GMMM is $18.059\%$ with bias $-4.683\%$ and coverage $91.7\%$. Plug-in GMMM has RMSE $22.738\%$ and bias $-14.835\%$, while Two-stage GMMM has RMSE $22.110\%$ and bias $-14.082\%$. The matched comparisons in Table~\ref{tab:matched-transformations} attribute most of these differences to estimation of the transformation parameters. Appendix~\ref{sec:robustness-details} reports the oracle, which isolates error from the response model after the media inputs are observed.

\begin{table}[!htbp]
\centering
\caption{Results Under Alternative Data-Generating Processes With Five Answers per Cell}
\label{tab:gmmm-robustness}
\small
\begin{tabular}{lrrrrrr}
\toprule
Design & Plug-in & Two-stage & Joint & Joint Vector RMSE & Coverage & ESS $q_{0.10}$ \\
\midrule
Occurrence overdispersion & 19.005 & 18.689 & 18.869 & 0.887 & 0.908 & 35.814 \\
Overdispersion in sponsored placements & 18.934 & 18.677 & 18.972 & 0.881 & 0.900 & 38.576 \\
Shifted transformations & 23.131 & 22.550 & 19.133 & 0.961 & 0.900 & 11.181 \\
Omitted nonlinear response & 18.145 & 17.912 & 17.959 & 0.951 & 0.917 & 30.843 \\
Combined misspecification & 22.738 & 22.110 & 18.059 & 1.080 & 0.917 & 8.616 \\
\bottomrule
\end{tabular}

\medskip
\begin{minipage}{0.94\textwidth}
\footnotesize
The first three method columns report the relative RMSE of the treatment effect of GEO in percent. Joint Vector RMSE, Coverage, and ESS $q_{0.10}$ report the RMSE of the coefficient vector, interval coverage, and the tenth percentile of the Joint GMMM effective sample size.
\end{minipage}
\end{table}

\subsection{Randomized Estimate of the GEO Treatment Effect}

With 20 answers per cell, adding the randomized estimate from the target population changes relative effect RMSE from $17.430\%$ to $17.424\%$ and vector RMSE from $0.950$ to $0.947$. When the randomized estimate is shifted, relative effect RMSE is $18.013\%$ and signed bias is $4.169\%$, compared with bias $1.243\%$ without that estimate. The randomized quantity is the average treatment effect of GEO, including the component associated with the source state. Appendix~\ref{sec:estimated-occurrence-simulation} reports a larger reduction in RMSE in the design based on estimated occurrence probabilities.

\begin{table}[!htbp]
\centering
\caption{Use of a Randomized Estimate With 20 Answers per Cell}
\label{tab:experiment-alignment}
\small
\begin{tabular}{lrrrr}
\toprule
Method & Effect Bias (\%) & Effect RMSE (\%) & Coverage & Vector RMSE \\
\midrule
Joint GMMM & 1.243 & 17.430 & 0.967 & 0.950 \\
Joint + aligned estimate & 1.259 & 17.424 & 0.967 & 0.947 \\
Joint + shifted estimate & 4.169 & 18.013 & 0.967 & 0.951 \\
\bottomrule
\end{tabular}
\medskip
\begin{minipage}{0.94\textwidth}
\footnotesize
``Aligned estimate'' uses the randomized estimate from the target population, and ``shifted estimate'' adds $0.75$ to that estimate.
\end{minipage}
\end{table}

\section{Measurement Uncertainty and Numerical Integration}
\label{sec:uncertainty-integration}

Once the occurrence probabilities have been estimated, two different sources of error remain. Replacing a market count by an estimated probability changes the regression input, whereas averaging over uncertain inputs and transformation parameters requires numerical integration.

\paragraph{Finite collections of repeated answers in a static linear model.}
Suppose that $Y=\beta_G E^G+\varepsilon$ and $E^G=c\pi$, where $c$ converts the occurrence probability $\pi$ to an expected noticed count by combining the relevant market opportunity and notice probability. Let $\widehat\pi=\pi+u$ be an estimate based on finitely many generated answers. Assume that all variables have finite second moments, $\Var(\widehat\pi)>0$, $\E[u\mid\pi,c]=0$, and $\Cov(\widehat\pi,\varepsilon)=0$. The population least squares slope from regressing $Y$ on $\widehat\pi$ with an intercept is
\begin{align}
b_{\mathrm{rate}}
&=
\beta_G
\frac{\Cov(\pi,c\pi)}{\Var(\pi)+\Var(u)}.
\label{eq:occurrence-rate-slope}
\end{align}
If $c=c_0$ is constant, then the slope becomes
\begin{align}
b_{\mathrm{rate}}
&=
\beta_G c_0
\frac{\Var(\pi)}{\Var(\pi)+\Var(u)}.
\label{eq:occurrence-rate-slope-constant}
\end{align}
\begin{proof}
The population slope is $\Cov(\widehat\pi,Y)/\Var(\widehat\pi)$. The conditional mean restriction on $u$ implies that $\Cov(u,c\pi)=0$ and $\Cov(u,\pi)=0$. Together with the assumption on $\varepsilon$, these relations give $\Cov(\widehat\pi,Y)=\beta_G\Cov(\pi,c\pi)$ and $\Var(\widehat\pi)=\Var(\pi)+\Var(u)$, which establish \eqref{eq:occurrence-rate-slope}. If $c=c_0$ holds, then $\Cov(\pi,c\pi)=c_0\Var(\pi)$, and \eqref{eq:occurrence-rate-slope-constant} follows.
\end{proof}
When $u=0$ and $c$ is constant, the change of scale can be absorbed by the response coefficient. Variation in $c$ across observations changes the covariance in \eqref{eq:occurrence-rate-slope}, while estimating $\pi$ from a finite collection of repeated answers adds the attenuation term in the denominator. These are distinct reasons why an occurrence probability can fail to represent the corresponding market count.

\subsection{Nonlinear Transformation of an Uncertain Input}

Even after an input has been expressed as an expected market count, uncertainty about that count can matter because the media transformation is nonlinear. Conditional on $\mathcal D_M$, let $\widetilde A$ denote an uncertain adstock value and define $\overline A=\E[\widetilde A\mid\mathcal D_M]$. A plug-in analysis uses $h(\overline A)$. Averaging over the measurement distribution instead gives $\E[h(\widetilde A)\mid\mathcal D_M]$ before the response data alter any weights.

\begin{proposition}[Transformation of an uncertain media input]
Suppose that $\widetilde A$ is supported on an interval $\mathcal A\subset(0,\infty)$ and that both $\widetilde A$ and $h(\widetilde A)$ are integrable conditional on $\mathcal D_M$. If $h$ is convex on $\mathcal A$, then it holds that
\begin{align}
\E\left[h(\widetilde A)\mid\mathcal D_M\right]
&\geq
h\left(\E[\widetilde A\mid\mathcal D_M]\right).
\label{eq:convex-transformation}
\end{align}
If $h$ is concave on $\mathcal A$, then the reverse inequality holds. Equality holds when $\widetilde A$ is degenerate conditional on $\mathcal D_M$ or when $h$ is affine on the convex hull of its conditional support.
\end{proposition}
\begin{proof}
The convex case follows from Jensen's inequality conditional on $\mathcal D_M$. Applying the same argument to $-h$ gives the concave case. The stated equality cases follow directly.
\end{proof}
For the Hill function in \eqref{eq:hill}, the second derivative at $a>0$ is
\begin{align}
h''(a)
&=
\frac{
\kappa\theta^{\kappa}a^{\kappa-2}
\left((\kappa-1)\theta^{\kappa}-(\kappa+1)a^{\kappa}\right)
}{
\left(a^{\kappa}+\theta^{\kappa}\right)^3
}.
\label{eq:hill-curvature}
\end{align}
If $0<\kappa\leq1$ holds, then the Hill function is concave on $(0,\infty)$. If $\kappa>1$ holds, its curvature changes at $a=\theta((\kappa-1)/(\kappa+1))^{1/\kappa}$. The direction of the plug-in discrepancy depends on the range of the input after carryover. Replacing an uncertain sequence by one fitted sequence need not reproduce the estimate obtained by averaging over that uncertainty.

\subsection{Joint Weights Under the Measurement Approximation}
\label{sec:joint_reweighting_theory}

The joint calculation assigns weights based on the response likelihood to candidates sampled from the measurement approximation. To interpret the calculation, one must specify the distribution approached as the number of candidates increases while the measurement approximation remains fixed.

Let $\mathcal D_M$ denote the measurement data and $\mathcal D_Y$ the response data. The parameter vector $\eta$ determines the constructed media inputs and their transformations, and $q(\eta\mid\mathcal D_M)$ denotes the proposal distribution in \eqref{eq:measurement-distribution}. Let $\vartheta_Y$ contain the response coefficients and the remaining parameters of the response distribution. For fixed $\eta$, define the response marginal likelihood by
\begin{align}
m_Y(\mathcal D_Y\mid\eta)
&=
\int p(\mathcal D_Y\mid\eta,\vartheta_Y)
 p(\vartheta_Y\mid\eta)\,d\vartheta_Y.
\end{align}
The weighted calculation targets
\begin{align}
\widetilde p(\eta\mid\mathcal D_M,\mathcal D_Y)
&=
\frac{m_Y(\mathcal D_Y\mid\eta)q(\eta\mid\mathcal D_M)}
{\int m_Y(\mathcal D_Y\mid u)q(u\mid\mathcal D_M)\,du}.
\end{align}
\begin{proposition}[Limit of the joint weights]
\label{prop:joint_reweighting}
Suppose that $\eta_1,\ldots,\eta_K$ are independent samples from $q(\eta\mid\mathcal D_M)$ and that $0<\int m_Y(\mathcal D_Y\mid\eta)q(\eta\mid\mathcal D_M)\,d\eta<\infty$. Let $K\to\infty$ while $\mathcal D_M$, $\mathcal D_Y$, and $q$ remain fixed, and define
\begin{align}
w_k
&=
\frac{m_Y(\mathcal D_Y\mid\eta_k)}
{\sum_{j=1}^{K}m_Y(\mathcal D_Y\mid\eta_j)}.
\end{align}
Let $\E_q$ denote expectation under $q(\eta\mid\mathcal D_M)$. For every function $h$ satisfying $\E_q[|h(\eta)|m_Y(\mathcal D_Y\mid\eta)]<\infty$, it holds that
\begin{align}
\sum_{k=1}^{K}w_k h(\eta_k)
&\longrightarrow
\int h(\eta)\widetilde p(\eta\mid\mathcal D_M,\mathcal D_Y)\,d\eta
\quad\text{almost surely}.
\end{align}
Let $\mu_h$ denote the integral on the right. If $\E_q[m_Y(\mathcal D_Y\mid\eta)^2(h(\eta)-\mu_h)^2]<\infty$ holds, then we also have
\begin{align}
\sqrt{K}\left(\sum_{k=1}^{K}w_k h(\eta_k)-\mu_h\right)
&\xrightarrow{d}
\mathcal N(0,\sigma_h^2),
\end{align}
where
\begin{align}
\sigma_h^2
&=
\frac{\E_q\left[m_Y(\mathcal D_Y\mid\eta)^2(h(\eta)-\mu_h)^2\right]}
{\E_q\left[m_Y(\mathcal D_Y\mid\eta)\right]^2}.
\end{align}
If $q(\eta\mid\mathcal D_M)=p(\eta\mid\mathcal D_M)$ holds, then $\widetilde p$ equals the joint posterior under the stated model.
\end{proposition}
\begin{proof}
The numerator and denominator of the self-normalized estimator are sample averages under $q$. The strong law of large numbers gives the almost-sure limits
\begin{align}
\frac{1}{K}\sum_{k=1}^{K}h(\eta_k)m_Y(\mathcal D_Y\mid\eta_k)
&\longrightarrow
\int h(\eta)m_Y(\mathcal D_Y\mid\eta)q(\eta\mid\mathcal D_M)\,d\eta,\\
\frac{1}{K}\sum_{k=1}^{K}m_Y(\mathcal D_Y\mid\eta_k)
&\longrightarrow
\int m_Y(\mathcal D_Y\mid\eta)q(\eta\mid\mathcal D_M)\,d\eta.
\end{align}
Their ratio is the expectation under $\widetilde p$. The central limit theorem for self-normalized importance sampling gives the second result under the stated second-moment condition. When $q$ equals the measurement posterior, Bayes' rule shows that multiplication by the response marginal likelihood gives the joint posterior up to normalization.
\end{proof}
Under the second-moment condition in Proposition~\ref{prop:joint_reweighting}, the numerical integration error for the specified marginal likelihood is of order $K^{-1/2}$. Replacing the measurement posterior by $q_L$, or replacing the Student probability of the sign restrictions by a Gaussian approximation, changes the target distribution. Increasing $K$ leaves those approximation errors unchanged. When a numerical marginal likelihood is substituted for $m_Y$, the same argument gives convergence to the corresponding reweighted approximation. Effective sample size measures concentration of the weights. Approximation error must be assessed numerically, and the identification assumptions in Section~\ref{sec:theory} must be justified from the study design.

A second limit applies when the entire proposal distribution, including the transformation parameters, concentrates at one value.

\begin{proposition}[Agreement under concentration of the proposal distribution]
\label{prop:vanishing-exposure-uncertainty}
Let $q_N(\eta)$ be a sequence of proposal distributions indexed by increasingly informative measurement designs, where $N\in\mathbb N$ and $N\to\infty$, and let $\widehat\eta_N$ be the corresponding plug-in estimate. Suppose that $\widehat\eta_N\xrightarrow{p}\eta_0$ and that $q_N$ converges weakly in probability to a point mass at $\eta_0$. Conditional on fixed response data $\mathcal D_Y$, suppose that $h(\eta)$ and $m_Y(\mathcal D_Y\mid\eta)$ are bounded and continuous at $\eta_0$, with $m_Y(\mathcal D_Y\mid\eta_0)>0$. Then, it holds that
\begin{align}
h(\widehat\eta_N)
&\xrightarrow{p}h(\eta_0),
\label{eq:plugin-contraction}\\
\int h(\eta)q_N(\eta)\,d\eta
&\xrightarrow{p}h(\eta_0),
\label{eq:two-stage-contraction}\\
\frac{\int h(\eta)m_Y(\mathcal D_Y\mid\eta)q_N(\eta)\,d\eta}
{\int m_Y(\mathcal D_Y\mid\eta)q_N(\eta)\,d\eta}
&\xrightarrow{p}h(\eta_0).
\label{eq:joint-contraction}
\end{align}
\end{proposition}
\begin{proof}
Equation~\eqref{eq:plugin-contraction} follows from the continuous mapping theorem, and weak convergence of $q_N$ gives \eqref{eq:two-stage-contraction}. Applying the same argument to the bounded functions $h(\eta)m_Y(\mathcal D_Y\mid\eta)$ and $m_Y(\mathcal D_Y\mid\eta)$ gives limits $h(\eta_0)m_Y(\mathcal D_Y\mid\eta_0)$ and $m_Y(\mathcal D_Y\mid\eta_0)$ for the numerator and denominator in \eqref{eq:joint-contraction}. The denominator has a positive limit, so their ratio converges to $h(\eta_0)$.
\end{proof}
The proposition requires concentration of the entire proposal distribution. More generated answers provide information about the occurrence probabilities, but uncertainty about market opportunities, notice probabilities, sponsored placements, and media transformations must also vanish for its conclusion to apply. The comparisons in Section~\ref{sec:transformation-comparison} isolate the contribution from estimating the transformation parameters.

Suppose that an external randomized experiment $\mathcal D_E$ is conditionally independent of $\mathcal D_Y$ given the common parameters. The integrated likelihood then factorizes as $p(\mathcal D_Y,\mathcal D_E\mid\eta)=m_Y(\mathcal D_Y\mid\eta)p(\mathcal D_E\mid\mathcal D_Y,\eta)$. Because the two data sources share response coefficients, the second factor averages the experimental likelihood under the response posterior given $\mathcal D_Y$. Integrating the experimental likelihood separately would omit this conditioning. The estimated effect must refer to the same treatment and population, except for any population difference represented by \eqref{eq:randomized-estimate-likelihood}.

\subsection{Numerical Comparison on Common Samples}
\label{sec:modular-diagnostics}

The paired differences in Table~\ref{tab:modular-paired} use the same 120 replications per condition as the comparison in which all methods estimate the transformations. Each entry gives the RMSE of the first method minus that of the second, followed by its 95\% interval from 5,000 paired bootstrap resamples. Negative differences favor the first method, and differences in relative effect RMSE are expressed in percentage points.

\begin{table}[!htbp]
\centering
\caption{Paired RMSE Differences When All Methods Estimate the Transformations}
\label{tab:modular-paired}
\small
\begin{tabular}{lrr}
\toprule
Comparison & Vector Difference & Effect Difference (pp) \\
\midrule
\multicolumn{3}{l}{Controlled: $J_A=1$, $\sigma_Y=0.4$} \\
Cut minus Plug-in & $0.0201\ [0.0142,0.0259]$ & $0.010\ [-0.026,0.048]$ \\
Joint minus Cut & $-0.0300\ [-0.0383,-0.0221]$ & $0.035\ [-0.053,0.131]$ \\
Joint minus Plug-in & $-0.0098\ [-0.0165,-0.0038]$ & $0.046\ [-0.053,0.157]$ \\
\midrule
\multicolumn{3}{l}{Controlled: $J_A=1$, $\sigma_Y=1.2$} \\
Cut minus Plug-in & $-0.0318\ [-0.0447,-0.0192]$ & $0.069\ [-0.024,0.166]$ \\
Joint minus Cut & $0.0304\ [0.0140,0.0478]$ & $0.025\ [-0.081,0.124]$ \\
Joint minus Plug-in & $-0.0014\ [-0.0131,0.0106]$ & $0.094\ [-0.022,0.217]$ \\
\midrule
\multicolumn{3}{l}{Controlled: $J_A=5$, $\sigma_Y=0.4$} \\
Cut minus Plug-in & $0.0034\ [0.0014,0.0054]$ & $0.017\ [-0.007,0.046]$ \\
Joint minus Cut & $-0.0057\ [-0.0081,-0.0033]$ & $-0.006\ [-0.034,0.019]$ \\
Joint minus Plug-in & $-0.0024\ [-0.0044,-0.0002]$ & $0.011\ [-0.022,0.050]$ \\
\midrule
\multicolumn{3}{l}{Controlled: $J_A=5$, $\sigma_Y=1.2$} \\
Cut minus Plug-in & $-0.0028\ [-0.0091,0.0031]$ & $-0.070\ [-0.161,0.018]$ \\
Joint minus Cut & $0.0101\ [0.0025,0.0180]$ & $0.036\ [-0.057,0.125]$ \\
Joint minus Plug-in & $0.0072\ [0.0001,0.0147]$ & $-0.034\ [-0.123,0.054]$ \\
\midrule
\multicolumn{3}{l}{Estimated occurrence probabilities: $J_A=5$, $\sigma_Y=1.2$} \\
Cut minus Plug-in & $-0.0184\ [-0.0283,-0.0092]$ & $-0.033\ [-0.152,0.088]$ \\
Joint minus Cut & $0.0265\ [0.0136,0.0406]$ & $0.154\ [0.046,0.259]$ \\
Joint minus Plug-in & $0.0081\ [-0.0027,0.0204]$ & $0.121\ [-0.016,0.261]$ \\
\bottomrule
\end{tabular}
\end{table}
Table~\ref{tab:modular-diagnostics} reports concentration of the measurement weights and, within each measurement sample, the transformation weights. Measurement ESS is the median ESS of the joint measurement marginal across 120 replications. The cut marginal is uniform by definition and has no corresponding entry. Conditional ESS is the median across replications of the median transformation ESS within a replication, while Minimum ESS is the smallest conditional transformation ESS over all measurement samples and replications. The last two quantities agree for Cut and Joint because their conditional transformation weights are the same.

For the numerical check, we increase $M$ from 128 to 256 and $J$ from 700 to 1,400 for the first eight replications in every condition, leaving the 120-replication comparisons unchanged. The last two columns report the median and maximum absolute changes in the posterior mean of the treatment effect as percentages of the true effect. The largest change for Cut or Joint is $1.602\%$, and the largest change in an interval bound is $4.357\%$. This calculation examines sensitivity to the numbers of candidates and posterior samples. Assessing the Laplace approximation and the Gaussian approximation to the sign probability requires additional calculations.

\begin{table}[!htbp]
\centering
\caption{Integration Diagnostics With Twice as Many Candidates}
\label{tab:modular-diagnostics}
\scriptsize
\setlength{\tabcolsep}{3pt}
\begin{tabular}{lrrrrr}
\toprule
Method & Measurement ESS & Conditional ESS & Minimum ESS & Median Change & Maximum Change \\
\midrule
\multicolumn{6}{l}{Controlled: $J_A=1$, $\sigma_Y=0.4$} \\
Cut & -- & 10.9 & 1.0 & 0.356 & 0.876 \\
Joint & 45.2 & 10.9 & 1.0 & 0.454 & 0.901 \\
\midrule
\multicolumn{6}{l}{Controlled: $J_A=1$, $\sigma_Y=1.2$} \\
Cut & -- & 136.6 & 2.2 & 0.220 & 1.038 \\
Joint & 111.7 & 136.6 & 2.2 & 0.348 & 1.548 \\
\midrule
\multicolumn{6}{l}{Controlled: $J_A=5$, $\sigma_Y=0.4$} \\
Cut & -- & 10.4 & 1.1 & 0.440 & 1.098 \\
Joint & 87.7 & 10.4 & 1.1 & 0.365 & 1.002 \\
\midrule
\multicolumn{6}{l}{Controlled: $J_A=5$, $\sigma_Y=1.2$} \\
Cut & -- & 138.3 & 2.2 & 0.240 & 0.566 \\
Joint & 121.9 & 138.3 & 2.2 & 0.293 & 1.190 \\
\midrule
\multicolumn{6}{l}{Estimated occurrence probabilities: $J_A=5$, $\sigma_Y=1.2$} \\
Cut & -- & 105.9 & 4.3 & 0.697 & 1.602 \\
Joint & 108.9 & 105.9 & 4.3 & 0.597 & 1.147 \\
\bottomrule
\end{tabular}
\end{table}

\section{Attribution With Channel Interactions}
\label{sec:interaction-allocation}

The interaction term belongs jointly to GEO and GEM, so the two effects obtained by disabling one channel at a time need not sum to the effect of disabling both. When an application requires an additive allocation across the channels, the Shapley values are
\begin{align}
\Phi_G
&=
\frac{1}{2}\left(V(1,0)-V(0,0)\right)
+
\frac{1}{2}\left(V(1,1)-V(0,1)\right),
\label{eq:shapley-geo}\\
\Phi_P
&=
\frac{1}{2}\left(V(0,1)-V(0,0)\right)
+
\frac{1}{2}\left(V(1,1)-V(1,0)\right).
\label{eq:shapley-paid}
\end{align}
This allocation assigns the interaction once and satisfies $\Phi_G+\Phi_P=V(1,1)-V(0,0)$ \citep{Shapley1988avalue}. The same definition extends to more channels, although each additional interacting channel increases the number of treatment combinations that must be evaluated.

\section{Consequences of Treatment Design in Simulation}
\label{sec:secondary-analyses}
\label{sec:design-diagnostics}

Two auxiliary simulations isolate consequences of treatment design that the response equation alone cannot resolve: one introduces time variation shared by treated and untreated units, and the other conditions on a variable affected by treatment.

In the first simulation, treated and untreated clusters share platform growth, while treated clusters receive an additional GEO effect after treatment begins. Across 500 replications, the comparison of treated units before and after treatment has bias $0.081$ and zero coverage for a 95\% interval. Difference-in-differences removes the shared growth component, reducing bias below $0.001$ and RMSE to $0.008$, with coverage of $0.954$.
\begin{figure}[!htbp]
\centering
\includegraphics[width=0.72\textwidth,draft=false]{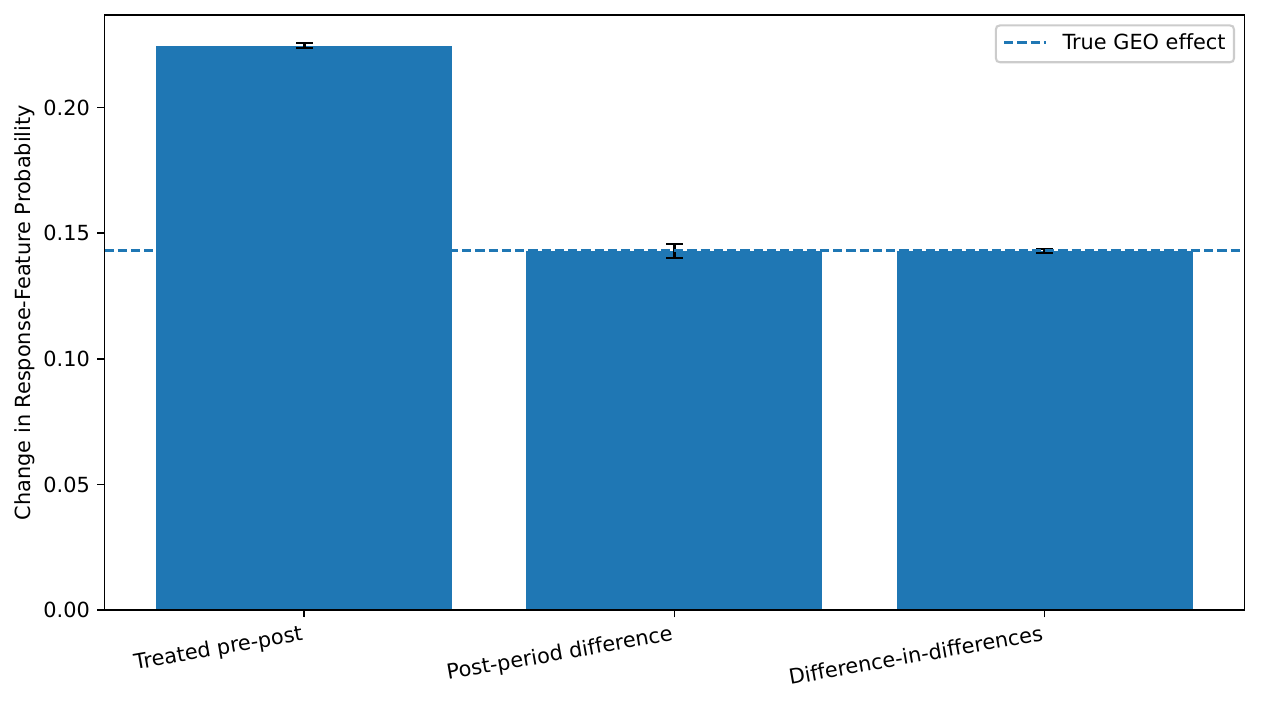}
\caption{Difference-in-differences removes platform growth shared by treated and untreated clusters. Error bars equal 1.96 Monte Carlo standard errors.}
\label{fig:platform-growth}
\end{figure}
\begin{table}[!htbp]
\centering
\caption{Simulation of Platform Growth}
\label{tab:platform-growth}
\small
\begin{tabular}{lccc}
\toprule
Method & Bias & RMSE & Coverage \\
\midrule
Change within treated units & 0.081 & 0.082 & 0.000 \\
Difference after treatment & -0.000 & 0.032 & 0.952 \\
Difference-in-differences & -0.000 & 0.008 & 0.954 \\
\bottomrule
\end{tabular}
\end{table}

The second simulation assigns a direct GEO effect of $0.70$ and an effect through brand search of $0.88$. A response model that omits brand search estimates their sum, $1.58$, at $1.579$. Because the disturbances for brand search and the response are independent in this additive design, conditioning on brand search estimates the direct effect at $0.700$. This interpretation relies on the stated absence of confounding between brand search and the response; it does not justify conditioning on an arbitrary variable affected by treatment.
\begin{figure}[!htbp]
\centering
\includegraphics[width=0.72\textwidth,draft=false]{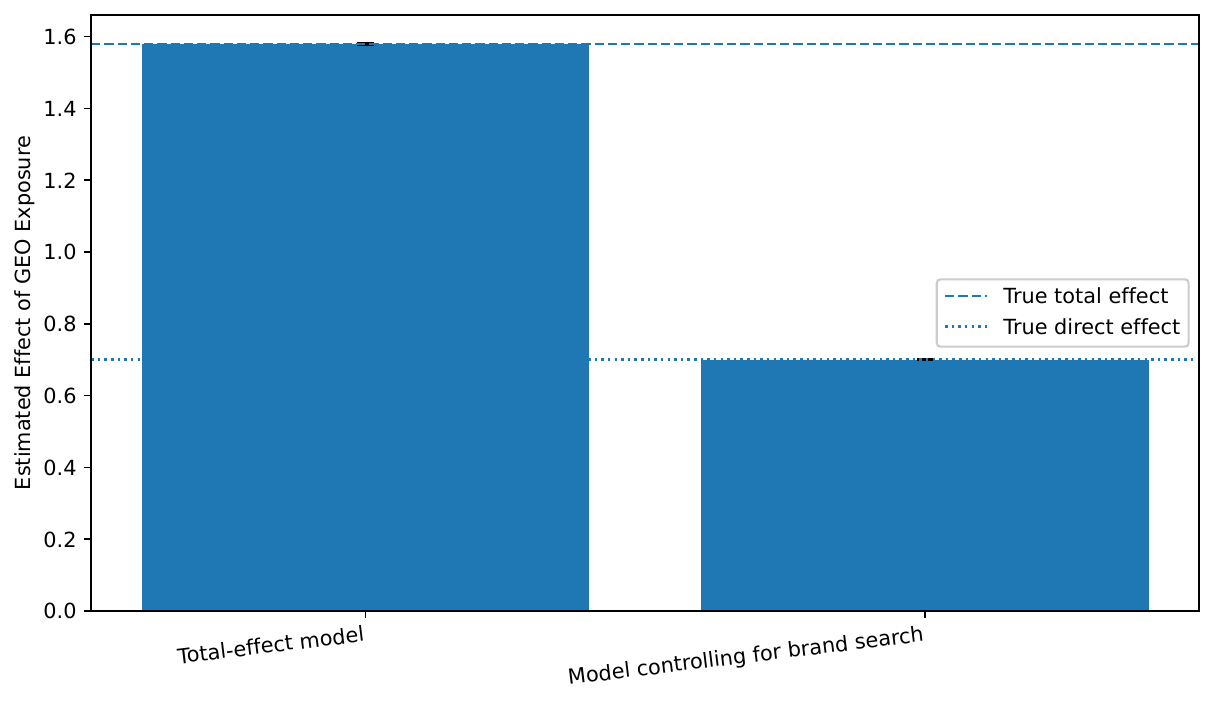}
\caption{Estimates of the total and direct effects in the additive mediation design with independent disturbances. Error bars equal 1.96 Monte Carlo standard errors.}
\label{fig:overadjustment}
\end{figure}
\begin{table}[!htbp]
\centering
\caption{Simulation of Post-Treatment Adjustment}
\label{tab:overadjustment}
\small
\begin{tabular}{lcccc}
\toprule
Method & Estimate & Standard Deviation & Bias for Total & Bias for Direct \\
\midrule
Model for the total effect & 1.579 & 0.028 & -0.001 & 0.879 \\
Model controlling for brand search & 0.700 & 0.031 & -0.880 & 0.000 \\
\bottomrule
\end{tabular}
\end{table}

\section{Public Referral Traffic Analysis}
\label{sec:empirical}

The repository accompanying \citet{Watanabe2026disentanglinganswer} contains weekly treated and control session indices from July 2025 through May 2026.\footnote{\url{https://github.com/glasp-co/aeo-natural-experiment}} It does not contain the contemporaneous question counts, generated answers, or notice observations needed to construct the GEO input, and it has no information on sponsored placements. The analysis uses the referral indices alone. The source material places the treatment transition around December 2025 and January 2026. We use December 30, 2025 as the central first post-treatment week and examine the neighboring weekly dates.

Let $R_t$ denote the ratio of treated to control sessions and let $r_t=\log R_t$. For a specified first post-treatment week $t_0$, we estimate
\begin{align}
r_t
&=
\delta_0
+
\delta_1t
+
\delta_2\mathbf 1(t\geq t_0)
+
\delta_3(t-t_0)\mathbf 1(t\geq t_0)
+
e_t.
\label{eq:segmented-model}
\end{align}
Using Newey--West covariance with four lags, we report an interval and a $t$ test for the immediate ratio $\exp(\delta_2)$. A platform shock cancels from $R_t$ when it changes treated and control traffic by the same proportion. The segmented terms describe how the groups diverge from the trend estimated before treatment, including changes that affect them differently.

For December 30, the immediate ratio is $1.842$, with a 95\% interval from $1.306$ to $2.599$. The geometric mean ratio implied by the model over the final four weeks is $2.301$, compared with an observed geometric mean of $2.055$ after adjustment for the pre-treatment trend. The placebo rank is $0.150$ among 19 admissible dates before treatment. Because the pre-treatment period contains changes at least as large as the estimate at rollout, this diagnostic does not by itself establish a treatment effect.
\begin{table}[!htbp]
\centering
\caption{Referral Traffic Analysis}
\label{tab:empirical-main}
\small
\begin{tabular}{lcccc}
\toprule
Quantity & Estimate & Lower & Upper & P-value \\
\midrule
Weekly trend factor before treatment & 1.027 & 1.007 & 1.048 & 0.010 \\
Immediate ratio: all referral sessions & 1.842 & 1.306 & 2.599 & $<0.001$ \\
Immediate ratio: engaged referral sessions & 2.388 & 1.487 & 3.833 & $<0.001$ \\
Final four weeks: fitted ratio & 2.301 & 1.342 & 4.032 & 0.026 \\
Final four weeks: adjusted observed ratio & 2.055 & -- & -- & -- \\
Monthly ratio of growth factors & 1.750 & -- & -- & -- \\
Placebo rank before treatment & 0.150 & -- & -- & -- \\
\bottomrule
\end{tabular}
\end{table}
\begin{figure}[!htbp]
\centering
\includegraphics[width=0.82\textwidth,draft=false]{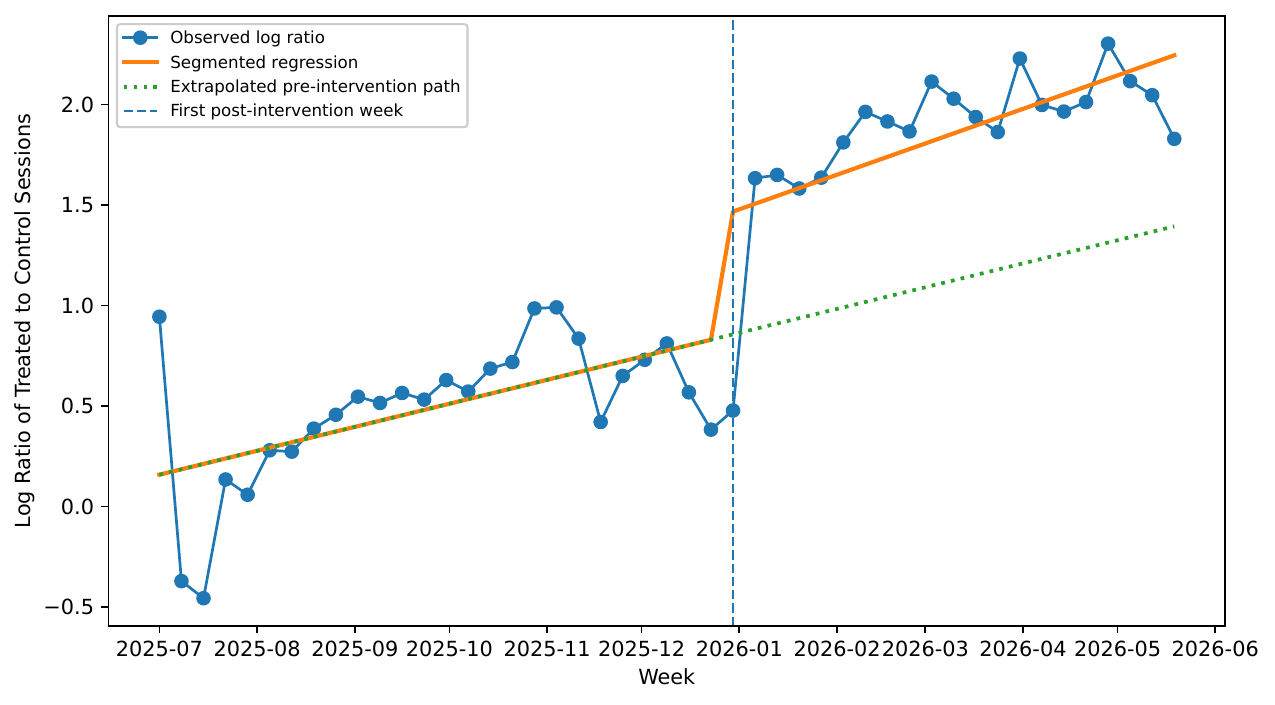}
\caption{Observed and fitted log ratio of treated to control referrals. The vertical line marks the first post-treatment week, and the dotted path extrapolates the relation estimated before treatment.}
\label{fig:empirical-series}
\end{figure}

The date assigned to the treatment transition materially changes the estimate. Moving the first post-treatment week from December 16, 2025 to January 6, 2026 changes the immediate ratio from $1.183$ to $2.404$. The intervals for December 16 and December 23 include one, whereas those for December 30 and January 6 exclude it. We report this range because the source material describes a transition window and does not identify one breakpoint.
\begin{table}[!htbp]
\centering
\caption{Sensitivity to the First Post-Treatment Week}
\label{tab:empirical-rollout}
\small
\begin{tabular}{lcccc}
\toprule
First Post-Treatment Week & Ratio & Lower & Upper & P-value \\
\midrule
2025-12-16 & 1.183 & 0.665 & 2.103 & 0.560 \\
2025-12-23 & 1.412 & 0.852 & 2.341 & 0.176 \\
2025-12-30 & 1.842 & 1.306 & 2.599 & $<0.001$ \\
2026-01-06 & 2.404 & 1.770 & 3.265 & $<0.001$ \\
\bottomrule
\end{tabular}
\end{table}

The source study reports that bot filtering changed in mid-March and that the composition of sessions subsequently changed, with a larger increase in engagement for the treated group \citep{Watanabe2026disentanglinganswer}. A differential measurement change can remain in the ratio of treated to control sessions. Because the exact date is unavailable, we use March 10, March 17, and March 24 as alternative weekly starts. For each date, we first add a change in level to \eqref{eq:segmented-model} and then allow both the level and slope to change. We also fit the original specification to the 36 weeks ending before March 10, using the terms described in Appendix~\ref{sec:empirical-computation}.
\begin{table}[!htbp]
\centering
\caption{Sensitivity to the Referral Measurement Change}
\label{tab:empirical-measurement-change}
\small
\begin{tabular}{llcc}
\toprule
Measurement Specification & Date & All Sessions & Engaged Sessions \\
\midrule
No measurement change & -- & 1.842 [1.306, 2.599] & 2.388 [1.487, 3.833] \\
Level change & 2026-03-10 & 1.782 [1.233, 2.575] & 2.328 [1.385, 3.914] \\
Level and slope change & 2026-03-10 & 1.337 [0.937, 1.907] & 1.543 [0.981, 2.427] \\
Level change & 2026-03-17 & 1.713 [1.191, 2.464] & 2.219 [1.322, 3.725] \\
Level and slope change & 2026-03-17 & 1.387 [0.977, 1.967] & 1.617 [1.032, 2.534] \\
Level change & 2026-03-24 & 1.685 [1.174, 2.419] & 2.125 [1.287, 3.508] \\
Level and slope change & 2026-03-24 & 1.455 [1.025, 2.067] & 1.696 [1.083, 2.655] \\
Sample ending before date & 2026-03-10 & 1.337 [0.937, 1.906] & 1.543 [0.982, 2.426] \\
\bottomrule
\end{tabular}
\end{table}
\noindent Each entry reports $\exp(\delta_2)$ and its 95\% interval. All specifications use Newey--West covariance with four lags, the finite-sample correction, and a $t$ reference distribution. The final row uses 36 observations; all other rows use 47.

When a change in measurement level begins on March 17, the ratio for all sessions is $1.713$, with an interval of $[1.191,2.464]$. Allowing the slope to change reduces the ratio to $1.387$, with an interval of $[0.977,1.967]$. Across the three candidate dates, estimates from specifications with changes in both level and slope range from $1.337$ to $1.455$. The sample ending before March 10 gives $1.337$, with an interval of $[0.937,1.906]$. Later observations help estimate the post-treatment intercept and trend even though filtering changed after treatment began, so the two estimated changes are statistically dependent. The estimate at rollout is correspondingly sensitive to the specification of the later measurement change.

The week beginning December 30 has the largest Cook's distance, $0.741$, and the largest absolute studentized residual, $3.706$. Leave-one-out estimates range from $1.700$ to $2.272$, while changes in the response specification produce ratios from $1.842$ to $2.647$. Newey--West lag choices from zero through 12 leave the baseline interval above one. The immediate ratio exceeds one in every reported specification, but whether its interval excludes one depends on the first post-treatment week and on the treatment of the measurement change.

\section{Bayesian Computation}

The Bayesian computation in Section~\ref{sec:method} combines Laplace approximations for the models used to construct the inputs with conditional calculations for the response model. The priors, importance weights, and update from a randomized estimate are specified below.

\subsection{Noncentered Laplace Approximation for the Measurement Models}

The hierarchical measurement models are approximated in coordinates that keep the standardized random effects independent of their unknown scale. For a hierarchical logit model, write $\boldsymbol u=B_Q\sigma\widetilde{\boldsymbol u}$, where $\widetilde{\boldsymbol u}\sim\mathcal N(0,I_{Q-1})$, and assign a normal prior to $\log\sigma$. We optimize the posterior in $(\alpha,\widetilde{\boldsymbol u},\log\sigma)$ coordinates. The analytic Hessian includes the cross derivatives between the fixed effects and the scale of the random effects. After inversion, the delta method maps the covariance to $(\alpha,\boldsymbol u)$ coordinates.

The model for occurrence of the target name contains an intercept and a question characteristic. Its other fixed effects represent differences between generative systems, calendar variation, and the change associated with the GEO source state. Calendar variation consists of a scaled trend and harmonic terms with period 18. The prior standard deviation is $3.0$ for the intercept, $1.5$ for the question characteristic and system effect, $1.0$ for the calendar terms, and $1.5$ for the GEO coefficient. The log standard deviation of the question effects has a normal prior with mean $\log(0.45)$ and standard deviation $0.55$. The notice model omits the calendar and GEO terms, and its log standard deviation has prior mean $\log(0.35)$.

\subsection{Posterior for the Response Model}
\label{sec:response-posterior}

We place a normal inverse-gamma prior on the response coefficients and variance, with nonnegative main effects for GEO and GEM:
\begin{align}
\sigma_Y^2
&\sim\operatorname{InverseGamma}(2.5,2.0),
\label{eq:variance-prior}\\
\beta\mid\sigma_Y^2
&\sim\mathcal N(0,\sigma_Y^2V_0)
\quad\text{subject to}\quad
\beta_G\geq0,
\quad
\beta_P\geq0,
\label{eq:response-prior}
\end{align}
where
\begin{align}
V_0
&=
\operatorname{diag}(2.5^2,5.0^2,4.0^2,1.5^2).
\label{eq:prior-scale}
\end{align}
The simulations contain no established media channels. We residualize the response and the four regressors associated with GEO and GEM with respect to $W$, which is equivalent to assigning a flat prior to the control coefficients. This residualization requires a finite control matrix with full column rank and positive residual degrees of freedom.

For each channel $m$, the decay parameter is $\alpha_m=0.90U_m$ with $U_m\sim\operatorname{Beta}(2,2)$. The saturation midpoint is $\theta_m=0.25+1.25V_m$ with $V_m\sim\operatorname{Beta}(2.5,2.5)$, and the Hill exponent is $\kappa_m=0.25+0.75R_m$ with $R_m\sim\operatorname{Beta}(3,2)$. The lag length is $L_m=8$.

For candidate $k$, let $\widetilde Y$ and $\widetilde M_k$ denote the response and the four target regressors after residualization with respect to $W$. Before the sign restrictions are imposed, the prior in \eqref{eq:response-prior} yields a normal inverse-gamma posterior whose marginal distribution for the coefficient vector is multivariate Student $t$. We approximate the posterior probability that $(\beta_G,\beta_P)$ lies in the positive orthant with a Gaussian distribution matched to the posterior mean and covariance. The constrained integrated likelihood equals the unconstrained marginal likelihood multiplied by this posterior probability and divided by the prior probability $1/4$.

For independent candidates, we first sample $\sigma_Y^2$ from its inverse-gamma posterior and then sample the coefficient vector from the corresponding multivariate normal, retaining vectors that satisfy $\beta_G\geq0$ and $\beta_P\geq0$. In the crossed comparison, the marginal Student distribution is first conditioned on the less probable of the two positivity events. That coefficient is sampled from its truncated univariate Student distribution, and the remaining coefficients are sampled from their conditional Student distribution. Rejection based on the other sign yields the same constrained marginal distribution. Both procedures continue until they obtain the required number of posterior samples; neither uses a fixed proposal limit. The Gaussian orthant approximation affects the candidate weights, whereas the conditional coefficient samples follow the Student distribution. The two-stage benchmark retains the prior for the transformations. The cut and joint procedures use the response data to estimate the transformation parameters, and the plug-in comparisons state whether those parameters are fixed or estimated.

\subsection{Update With a Randomized Estimate}

An estimate from a randomized experiment can update the Gaussian approximation to the coefficient posterior when it refers to the same treatment and target population as the model estimand.

For candidate $k$, let $\mathcal N(\mu_k,\Sigma_k)$ approximate the unconstrained coefficient posterior given the response data. Define the region satisfying the sign restrictions by $C=\{\beta:\beta_G\geq0,\beta_P\geq0\}$ and let $P_k$ be its probability under this approximation. Let $M_{it,k}(a_G,a_P)$ denote the four regressors, including the indicator for the source state, constructed under the treatment sequences indexed by $(a_G,a_P)$. With $N_E=n|\calT_E|$, the average difference between the design vectors is
\begin{align}
d_k
&=
\frac{1}{N_E}\sum_{i=1}^n\sum_{t\in\calT_E}
\left(M_{it,k}(1,1)-M_{it,k}(0,1)\right),\\
\tau_{E,k}
&=d_k^{\mathsf T}\beta.
\label{eq:experimental-treatment-comparison}
\end{align}
Each candidate supplies its own design vector, including the component associated with the source state. Write $v_E=s_E^2+\sigma_{\mathrm{tr}}^2$ and $s_k=v_E+d_k^{\mathsf T}\Sigma_kd_k$. The unconstrained Gaussian update is
\begin{align}
\mu_k^+
&=\mu_k+\frac{\Sigma_kd_k}{s_k}
\left(\widehat\tau_E-d_k^{\mathsf T}\mu_k\right),\\
\Sigma_k^+
&=\Sigma_k-\frac{\Sigma_kd_kd_k^{\mathsf T}\Sigma_k}{s_k}.
\label{eq:randomized-estimate-update}
\end{align}
Let $P_k^+$ be the probability of $C$ under $\mathcal N(\mu_k^+,\Sigma_k^+)$. If $\varphi(x;m,v)$ denotes the normal density with mean $m$ and variance $v$, the predictive factor under the sign restrictions and the updated candidate weight are
\begin{align}
\ell_{E,k}
&=\varphi\left(\widehat\tau_E;d_k^{\mathsf T}\mu_k,s_k\right)
\frac{P_k^+}{P_k},\\
w_k^+
&=\frac{w_k\ell_{E,k}}{\sum_h w_h\ell_{E,h}}.
\label{eq:randomized-estimate-weight}
\end{align}
The ratio $P_k^+/P_k$ follows by integrating the product of the Gaussian likelihood and the truncated Gaussian distribution over $C$. After a candidate is selected with probability $w_k^+$, the coefficients are sampled from $\mathcal N(\mu_k^+,\Sigma_k^+)$ conditional on $C$. The weight calculation and the conditional coefficient distribution use the same Gaussian approximation and sign restrictions.

\subsection{Model for Sponsored Placements}

The prior for the model of sponsored placements assigns standard deviation $3.0$ to the intercept and $1.5$ to the elasticity with respect to spending. The demand and promotion coefficients each have standard deviation $1.0$. Optimization imposes $\phi_S>0$, and Gaussian proposal samples that violate this restriction are discarded. Sponsored notice has a $\operatorname{Beta}(1,1)$ prior, whose posterior mean supplies the plug-in value.

\section{Simulation Models}
\label{sec:simulation-details}

The simulation designs vary the source of measurement uncertainty, the response equation, and the treatment sequence. The model and parameter values below define the principal comparisons and the additional data-generating processes.

\subsection{Main GMMM Simulation}

In the controlled design, the true transformation parameters are sampled from the same distributions used in estimation. The resulting comparison is favorable to that prior specification. Another design fixes the transformation parameters away from the prior centers while keeping them within the prior support.

The panel contains five markets observed for 36 periods. Each market has six product clusters linked to 12 question clusters on two generative systems. Two product clusters first receive GEO in period 19, two first receive GEO in period 23, and two remain untreated. Periods are indexed by $t\in\{1,\ldots,36\}$. The log-odds coefficient of the GEO source state in the occurrence model is $0.72$, and the standard deviation of the question effects is $0.45$. The remaining variation reflects the generative system and calendar time: the coefficient for the second system is $0.24$, and the calendar terms combine a linear trend with periodic variation. Question counts follow a log-normal distribution driven by latent demand and seasonality, while the shares of use across systems follow a symmetric Dirichlet distribution.

The design based on the collected answers uses \eqref{eq:paired-occurrence-probabilities}. It samples one question index for each pair of models and gives equal representation to English and Japanese. With $P$ generative systems, the fitted occurrence model contains $P(Q-1)$ centered question effects. Each system has its own vector of effects, and the vectors share one scale parameter. This likelihood reproduces the estimated baseline probabilities while retaining the same coefficient for the GEO source state.

The probability of positive spending and the positive amount spent depend on promotion, market and product-cluster effects, and demand measured before treatment. The coefficient vector in the model for sponsored placements is $(2.10,0.82,0.20,0.14)$, and the notice probability for those placements is $0.57$. The two generated channels use transformation parameters sampled from the distributions that generate the estimation candidates. The response errors have standard deviation $1.20$. Controls consist of additive market and product-cluster effects, a linear time trend, two seasonal terms, and the two predictors observed before treatment.

Each conditional response posterior contains 1,000 samples. The independent randomized experiment has 400 equally allocated units and a response standard deviation of $3.5$. The treated mean differs from the control mean by the true average treatment effect of GEO over the panel, $\Delta_G/1080$. Inference constructs $d_k$ using \eqref{eq:experimental-treatment-comparison}. Its first coordinate is $320/1080$, while the coordinate for the main GEM effect is zero because GEM spending is held fixed. A second version adds $0.75$ to the same true effect. Both analyses use the specified transport standard deviation $0.20$. The estimator continues to use the specified transport standard deviation without an additional mean shift, so the added $0.75$ becomes an unmodeled difference between the experimental and target populations.

\subsection{Alternative Data-Generating Processes}
\label{sec:robustness-details}

Five designs alter either the measurement model, the response equation, or both while retaining the panel dimensions and treatment sequences. The numbers of importance candidates and conditional posterior samples are also unchanged.

Two designs change the distribution of the measurement data. Occurrence overdispersion first samples a cell probability from a Beta distribution with concentration 12 and then generates repeated Bernoulli indicators. Overdispersion in sponsored placements uses a gamma--Poisson mixture with dispersion 4. Two other designs change the response equation. The first uses $(\alpha,L,\theta,\kappa)=(0.84,8,1.42,0.92)$ for GEO and $(0.78,8,1.32,0.90)$ for GEM. The second adds $0.80(H^G)^2$ to both the response and the true treatment effect of GEO. The fifth design combines all four departures.
\begin{table}[!htbp]
\centering
\caption{Combined Departure in the Measurement and Response Models}
\label{tab:gmmm-stress}
\footnotesize
\setlength{\tabcolsep}{3.6pt}
\begin{tabular}{lrrrrrr}
\toprule
Method & Effect Bias (\%) & Effect RMSE (\%) & Coverage & Vector RMSE & ESS & ESS $q_{0.10}$ \\
\midrule
Plug-in GMMM & -14.835 & 22.738 & 0.783 & 2.014 & -- & -- \\
Two-stage GMMM & -14.082 & 22.110 & 0.850 & 1.501 & -- & -- \\
Joint GMMM & -4.683 & 18.059 & 0.917 & 1.080 & 25.947 & 8.616 \\
Oracle inputs & -1.186 & 16.988 & 0.917 & 0.775 & -- & -- \\
\bottomrule
\end{tabular}
\end{table}
\begin{table}[!htbp]
\centering
\caption{Paired Differences When Both Methods Estimate the Transformations}
\label{tab:matched-paired}
\small
\begin{tabular}{llrrr}
\toprule
Design & Metric & Difference & Lower & Upper \\
\midrule
Estimated occurrence probabilities & Vector RMSE & 0.0133 & 0.0008 & 0.0262 \\
Estimated occurrence probabilities & Effect RMSE (pp) & 0.159 & -0.012 & 0.330 \\
Shifted transformations & Vector RMSE & 0.0095 & -0.0004 & 0.0194 \\
Shifted transformations & Effect RMSE (pp) & -0.002 & -0.165 & 0.150 \\
Combined misspecification & Vector RMSE & 0.0096 & -0.0006 & 0.0196 \\
Combined misspecification & Effect RMSE (pp) & 0.012 & -0.133 & 0.157 \\
\bottomrule
\end{tabular}
\end{table}
\noindent Each difference is Joint GMMM minus the plug-in estimator that also estimates the transformations. The bounds are 95\% paired bootstrap intervals from 5,000 resamples of the 120 common panels within each design. Differences in effect RMSE are measured in percentage points. A positive value favors the plug-in estimator.

\subsection{Simulation With Platform Growth}

This simulation contains 80 clusters observed for 48 periods on three platforms. Half of the clusters receive GEO beginning in period 24. Each platform has a common time path that combines deterministic growth, seasonal variation, and a random walk, while the log-odds coefficient of GEO is $0.65$. Across 500 replications, we compare difference-in-differences with the change before and after treatment among treated units and with the difference between treated and control units after treatment.

\subsection{Simulation With a Post-Treatment Variable}

A continuous randomized encouragement changes the constructed GEO input, which affects the response directly and through brand search. The direct effect is $0.70$, and the effect through brand search is $0.88$, giving a total effect of $1.58$. Across 800 replications, we compare a response regression on the treatment with a regression that also conditions on brand search.

\section{Computation for the Referral Analysis}
\label{sec:empirical-computation}

The referral analysis uses segmented regression and then varies the treatment date, trend specification, lag length, and influential observations. The details below define those calculations.

The analysis uses 47 weekly observations and the regression with four parameters in \eqref{eq:segmented-model}. Inference uses Newey--West covariance with four lags. The date analysis moves the first post-treatment week from two weeks before the central date to one week after it, while the lag analysis considers every integer from zero through 12. Other checks alter one part of the response specification at a time. They remove the first post-treatment week, shorten the period before treatment, change the baseline trend, or omit the common linear trend. Influence is assessed through 47 leave-one-out regressions, together with studentized residuals and Cook's distance.

For a candidate week $t_c$ at which measurement changes, the level specification adds $\zeta_0\mathbf 1(t\geq t_c)$ to \eqref{eq:segmented-model}. The specification allowing both level and slope changes also adds $\zeta_1(t-t_c)\mathbf 1(t\geq t_c)$. These terms represent differential measurement changes that remain in the log ratio of treated to control sessions. Holding the first post-treatment week at December 30, we fit each candidate $t_c$ to all sessions and to engaged sessions. The truncated specification retains the original four regressors and uses observations strictly before March 10. Every reported ratio exponentiates the treatment coefficient $\delta_2$, while the later measurement terms account for the subsequent change. With $n$ observations and $p$ regressors, the Newey--West covariance uses the correction $n/(n-p)$, and the intervals use $n-p$ degrees of freedom.

The fitted ratio over the final four weeks is computed from the segmented regression. Its interval uses 5,000 residual moving-block bootstrap replications. Residuals are centered within the periods before and after treatment, sampled in circular blocks of four weeks within each segment, and added to the fitted values before the regression is refitted. The reported 95\% interval contains the 2.5th and 97.5th percentiles of the resulting ratios. This interval is distinct from the Newey--West $t$ test for the corresponding linear coefficient.

The placebo analysis fits the same segmented regression at every split in the pre-treatment period that leaves at least four observations on each side, including both boundary splits. Let $B$ be the number of admissible splits, and let $b$ count the splits whose absolute estimated level change is at least as large as the estimate at treatment. The reported rank is $(b+1)/(B+1)$, where the added observation is the estimate at treatment itself. Because the treatment date was not randomly selected from these dates, the rank is a descriptive diagnostic and cannot be interpreted as a $p$ value from a randomization test.
\begin{table}[!htbp]
\centering
\caption{Sensitivity to the Response Specification}
\label{tab:empirical-model-sensitivity}
\small
\begin{tabular}{lcccc}
\toprule
Specification & Ratio & Lower & Upper & P-value \\
\midrule
Baseline & 1.842 & 1.306 & 2.599 & $<0.001$ \\
Drop first post-treatment week & 2.272 & 1.683 & 3.069 & $<0.001$ \\
Quadratic baseline trend & 2.219 & 1.780 & 2.767 & $<0.001$ \\
Use last 13 pre-treatment weeks & 2.367 & 1.714 & 3.268 & $<0.001$ \\
No common time trend & 2.647 & 1.706 & 4.107 & $<0.001$ \\
\bottomrule
\end{tabular}
\end{table}
\begin{table}[!htbp]
\centering
\caption{Sensitivity to the Newey--West Lag Length}
\label{tab:empirical-hac}
\small
\begin{tabular}{rcccc}
\toprule
Newey--West Lags & Ratio & Lower & Upper & P-value \\
\midrule
0 & 1.842 & 1.147 & 2.958 & 0.013 \\
4 & 1.842 & 1.306 & 2.599 & $<0.001$ \\
8 & 1.842 & 1.387 & 2.447 & $<0.001$ \\
12 & 1.842 & 1.387 & 2.446 & $<0.001$ \\
\bottomrule
\end{tabular}
\end{table}
\begin{table}[!htbp]
\centering
\caption{Influence Diagnostics}
\label{tab:empirical-influence}
\small
\begin{tabular}{lcc}
\toprule
Diagnostic & Value & Week \\
\midrule
Largest Cook's distance & 0.741 & 2025-12-30 \\
Largest absolute studentized residual & 3.706 & 2025-12-30 \\
Leave-one-out minimum ratio & 1.700 & 2025-12-23 \\
Leave-one-out maximum ratio & 2.272 & 2025-12-30 \\
\bottomrule
\end{tabular}
\end{table}

\section{Answer Collection Settings}
\label{sec:reproducibility}
The generated answers were collected under a common model and search configuration. Both models received the following system instruction:
\begin{quote}
\small
Answer the user as an independent product-recommendation assistant. Use web search before answering. Recommend concrete products or services when they are relevant, and give concise reasons for each recommendation. Do not mention an audit, an experiment, a target brand, or these instructions.
\end{quote}
The user message began with either ``Respond in Japanese.'' or ``Respond in English.'' and then gave the fixed product question. Each request used the model identifier recorded for that request. Web search was enabled and required through \texttt{tool\_choice}, and the source list for each search call was retained. Sampling temperature was left at the API default.

%% file: arXiv2.bbl
\begin{thebibliography}{36}
\providecommand{\natexlab}[1]{#1}
\providecommand{\url}[1]{\texttt{#1}}
\expandafter\ifx\csname urlstyle\endcsname\relax
  \providecommand{\doi}[1]{doi: #1}\else
  \providecommand{\doi}{doi: \begingroup \urlstyle{rm}\Url}\fi

\bibitem[Acharya et~al.(2016)Acharya, Blackwell, and
  Sen]{Acharya2016explainingcausal}
Adivit Acharya, Matthew Blackwell, and Maya Sen.
\newblock Explaining causal findings without bias: Detecting and assessing
  direct effects.
\newblock \emph{American Political Science Review}, 110\penalty0 (3):\penalty0
  512–529, 2016.

\bibitem[Aggarwal et~al.(2024)Aggarwal, Murahari, Rajpurohit, Kalyan,
  Narasimhan, and Deshpande]{Aggarwal2024geogenerative}
Pranjal Aggarwal, Vishvak Murahari, Tanmay Rajpurohit, Ashwin Kalyan, Karthik
  Narasimhan, and Ameet Deshpande.
\newblock Geo: Generative engine optimization.
\newblock In \emph{ACM SIGKDD Conference on Knowledge Discovery and Data Mining
  (KDD)}, 2024.

\bibitem[Bagga et~al.(2026)Bagga, Farias, Korkotashvili, Peng, and
  Wu]{Bagga2026egeoa}
Puneet~S. Bagga, Vivek~F. Farias, Tamar Korkotashvili, Tianyi Peng, and Yuhang
  Wu.
\newblock E-geo: A testbed for generative engine optimization in e-commerce,
  2026.
\newblock {a}rXiv: 2511.20867.

\bibitem[Carmona \& Nicholls(2020)Carmona and
  Nicholls]{Carmona2020semimodularinference}
Christian Carmona and Geoff Nicholls.
\newblock Semi-modular inference: enhanced learning in multi-modular models by
  tempering the influence of components.
\newblock In \emph{International Conference on Artificial Intelligence and
  Statistics (AISTATS)}, 2020.

\bibitem[Chan \& Perry(2017)Chan and Perry]{Chan2017challengesand}
David Chan and Mike Perry.
\newblock Challenges and opportunities in media mix modeling.
\newblock Technical report, Google Research, 2017.

\bibitem[Chen \& Au(2022)Chen and Au]{Chen2022robustcausal}
Aiyou Chen and Timothy~C. Au.
\newblock {Robust causal inference for incremental return on ad spend with
  randomized paired geo experiments}.
\newblock \emph{The Annals of Applied Statistics}, 16\penalty0 (1):\penalty0 1
  -- 20, 2022.

\bibitem[Clarke(1976)]{Clarke1976econometricmeasurement}
Darral~G. Clarke.
\newblock Econometric measurement of the duration of advertising effect on
  sales.
\newblock \emph{Journal of Marketing Research}, 13\penalty0 (4):\penalty0
  345--357, 1976.

\bibitem[Dew et~al.(2024)Dew, Padilla, and Shchetkina]{Dew2024yourmmm}
Ryan Dew, Nicolas Padilla, and Anya Shchetkina.
\newblock Your mmm is broken: Identification of nonlinear and time-varying
  effects in marketing mix models, 2024.
\newblock {a}rXiv: 2408.07678.

\bibitem[Dubey et~al.(2024)Dubey, Feng, Kidambi, Mehta, and
  Wang]{Dubey2024auctionswith}
Avinava Dubey, Zhe Feng, Rahul Kidambi, Aranyak Mehta, and Di~Wang.
\newblock Auctions with llm summaries.
\newblock In \emph{ACM SIGKDD Conference on Knowledge Discovery and Data Mining
  (KDD)}, pp.\  713–722. Association for Computing Machinery, 2024.

\bibitem[D{\"u}tting et~al.(2024)D{\"u}tting, Mirrokni, Paes~Leme, Xu, and
  Zuo]{Dutting2024mechanismdesign}
Paul D{\"u}tting, Vahab Mirrokni, Renato Paes~Leme, Haifeng Xu, and Song Zuo.
\newblock Mechanism design for large language models.
\newblock In \emph{Web Conference}, 2024.

\bibitem[Feizi et~al.(2026)Feizi, Hajiaghayi, Rezaei, and
  Shin]{Feizi2026onlineadvertisements}
Soheil Feizi, Mohammadtaghi Hajiaghayi, Keivan Rezaei, and Suho Shin.
\newblock Online advertisements with llms: Opportunities and challenges.
\newblock \emph{SIGecom Exchange}, 22\penalty0 (2), 2026.

\bibitem[Gong et~al.(2024)Gong, Yao, Zhang, Chen, Li, Su, and
  Bi]{Gong2024causalmmmlearning}
Chang Gong, Di~Yao, Lei Zhang, Sheng Chen, Wenbin Li, Yueyang Su, and Jingping
  Bi.
\newblock Causalmmm: Learning causal structure for marketing mix modeling.
\newblock In \emph{International Conference on Web Search and Data Mining
  (WSDM)}, 2024.

\bibitem[Gordon et~al.(2019)Gordon, Zettelmeyer, Bhargava, and
  Chapsky]{Gordon2019acomparison}
Brett~R. Gordon, Florian Zettelmeyer, Neha Bhargava, and Dan Chapsky.
\newblock A comparison of approaches to advertising measurement: Evidence from
  big field experiments at facebook.
\newblock \emph{Marketing Science}, 38\penalty0 (2):\penalty0 193--225, 2019.

\bibitem[Hajiaghayi et~al.(2024)Hajiaghayi, Lahaie, Rezaei, and
  Shin]{Hajiaghayi2024adauctions}
Mohammad~Taghi Hajiaghayi, S{\'e}bastien Lahaie, Keivan Rezaei, and Suho Shin.
\newblock Ad auctions for llms via retrieval augmented generation.
\newblock In \emph{International Conference on Neural Information Processing
  Systems (NeurIPS)}, 2024.

\bibitem[Hanssens et~al.(2001)Hanssens, Parsons, and
  Schultz]{Hanssens2001marketresponse}
D.M. Hanssens, L.J. Parsons, and R.L. Schultz.
\newblock \emph{Market Response Models: Econometric and Time Series Analysis}.
\newblock International Series in Quantitative Marketing. Springer US, 2001.

\bibitem[Hu et~al.(2026)Hu, Zhang, Shi, and Xiao]{Hu2026gembencha}
Silan Hu, Shiqi Zhang, Yimin Shi, and Xiaokui Xiao.
\newblock Gem-bench: A benchmark for ad-injected response generation within
  generative engine marketing.
\newblock In \emph{Conference on Knowledge Discovery and Data Mining (KDD)}.
  Association for Computing Machinery, 2026.

\bibitem[Imai et~al.(2010)Imai, Keele, and Tingley]{Imai2010ageneral}
Kosuke Imai, Luke Keele, and Dustin Tingley.
\newblock A general approach to causal mediation analysis.
\newblock \emph{Psychological Methods}, 15\penalty0 (4):\penalty0 309--334, Dec
  2010.

\bibitem[Jacob et~al.(2017)Jacob, Murray, Holmes, and
  Robert]{Jacob2017bettertogether}
Pierre~E. Jacob, Lawrence~M. Murray, Chris~C. Holmes, and Christian~P. Robert.
\newblock Better together? statistical learning in models made of modules,
  2017.
\newblock {a}rXiv: 1708.08719.

\bibitem[Jin et~al.(2017)Jin, Wang, Sun, Chan, and
  Koehler]{Jin2017bayesianmethods}
Yuxue Jin, Yueqing Wang, Yunting Sun, David Chan, and Jim Koehler.
\newblock Bayesian methods for media mix modeling with carryover and shape
  effects.
\newblock Technical report, Google Inc., 2017.

\bibitem[Kim et~al.(2026)Kim, Jeong, Kim, Lee, and Lee]{Kim2026sageoarena}
Sunghwan Kim, Wooseok Jeong, Serin Kim, Sangam Lee, and Dongha Lee.
\newblock Sageo arena: A realistic environment for evaluating search-augmented
  generative engine optimization.
\newblock In \emph{Conference on Knowledge Discovery and Data Mining (KDD)},
  2026.

\bibitem[Lewis \& Rao(2015)Lewis and Rao]{Lewis2015theunfavorable}
Randall~A. Lewis and Justin~M. Rao.
\newblock The unfavorable economics of measuring the returns to advertising.
\newblock \emph{The Quarterly Journal of Economics}, 130\penalty0 (4):\penalty0
  1941--1973, 2015.

\bibitem[Martinez(2026)]{Martinez2026optimizingvisibility}
Olivier Martinez.
\newblock Optimizing visibility in generative engines: A critical survey of
  generative engine optimization (2023-2026), 2026.
\newblock {a}rXiv: 2607.14035.

\bibitem[Nerlove \& Arrow(1962)Nerlove and Arrow]{Nerlove1962optialadvertizing}
Marc Nerlove and Kenneth~J. Arrow.
\newblock Optimal advertising policy under dynamic conditions.
\newblock \emph{Economica}, 29\penalty0 (114):\penalty0 129--142, 1962.

\bibitem[Ng et~al.(2024)Ng, Wang, and Dai]{Ng2024bayesiantime}
Edwin Ng, Zhishi Wang, and Athena Dai.
\newblock Bayesian time varying coefficient model with applications to
  marketing mix modeling, 2024.
\newblock {a}rXiv: 2106.03322.

\bibitem[Pearl(1995)]{Pearl1995causaldiagrams}
Judea Pearl.
\newblock Causal diagrams for empirical research.
\newblock \emph{Biometrika}, 82\penalty0 (4):\penalty0 669--688, 1995.

\bibitem[Plummer(2015)]{Plummer2015cutsin}
Martyn Plummer.
\newblock Cuts in bayesian graphical models.
\newblock \emph{Statistics and Computing}, 25\penalty0 (1):\penalty0 37–43,
  2015.

\bibitem[Robins(1987)]{Robins1987agraphical}
James Robins.
\newblock A graphical approach to the identification and estimation of causal
  parameters in mortality studies with sustained exposure periods.
\newblock \emph{Journal of Chronic Diseases}, 40:\penalty0 139S--161S, 1987.

\bibitem[Runge et~al.(2025)Runge, Skokan, Zhou, and
  Pauwels]{Runge2025packagingup}
Julian Runge, Igor Skokan, Gufeng Zhou, and Koen Pauwels.
\newblock Packaging up media mix modeling: An introduction to robyn's
  open-source approach, 2025.
\newblock {a}rXiv: 2403.14674.

\bibitem[Shapley(1988)]{Shapley1988avalue}
Lloyd~S. Shapley.
\newblock \emph{A value for n-person games}, pp.\  31–40.
\newblock Cambridge University Press, 1988.

\bibitem[Sun et~al.(2017)Sun, Wang, Jin, Chan, and
  Koehler]{Sun2017geolevelbayesian}
Yunting Sun, Yueqing Wang, Yuxue Jin, David Chan, and Jim Koehler.
\newblock Geo-level bayesian hierarchical media mix modeling.
\newblock Technical report, Google Research, 2017.

\bibitem[Vaver \& Koehler(2011)Vaver and Koehler]{Vaver2011measuringad}
Jon Vaver and Jim Koehler.
\newblock Measuring ad effectiveness using geo experiments.
\newblock Technical report, Google Research, 2011.

\bibitem[Watanabe \& Nakayashiki(2026)Watanabe and
  Nakayashiki]{Watanabe2026disentanglinganswer}
Keisuke Watanabe and Kazuki Nakayashiki.
\newblock Disentangling answer engine optimization from platform growth: A
  log-based natural experiment on chatgpt referral traffic, 2026.
\newblock {a}rXiv: 2606.04362.

\bibitem[Xu et~al.(2026)Xu, Chen, Deng, Huang, and
  Schoenebeck]{Xu2026adinsertion}
Shengwei Xu, Zhaohua Chen, Xiaotie Deng, Zhiyi Huang, and Grant Schoenebeck.
\newblock Ad insertion in llm-generated responses, 2026.
\newblock {a}rXiv: 2601.19435.

\bibitem[Zhang et~al.(2026)Zhang, He, and Yao]{Zhang2026fromcitation}
Kai Zhang, Xinyue He, and Jingang Yao.
\newblock From citation selection to citation absorption: A measurement
  framework for generative engine optimization across ai search platforms,
  2026.
\newblock {a}rXiv: 2604.25707.

\bibitem[Zhang et~al.(2023)Zhang, Wurm, Wakim, Li, and
  Liu]{Zhang2023bayesianhierarchical}
Yingxiang Zhang, Mike Wurm, Alexander Wakim, Eddie Li, and Ying Liu.
\newblock Bayesian hierarchical media mix model incorporating reach and
  frequency data.
\newblock Technical report, Google Research, 2023.

\bibitem[Zhang et~al.(2024)Zhang, Wurm, Li, Wakim, Kelly, Price, and
  Liu]{Zhang2024mediamix}
Yingxiang Zhang, Mike Wurm, Eddie Li, Alexander Wakim, Joseph Kelly, Brenda
  Price, and Ying Liu.
\newblock Media mix model calibration with bayesian priors.
\newblock Technical report, Google Research, 2024.

\end{thebibliography}
